%% file: hmc.tex
\documentclass{article}

\usepackage{arxiv}

\usepackage[utf8]{inputenc} 
\usepackage[T1]{fontenc}    
\usepackage{hyperref}       
\usepackage{url}            
\usepackage{booktabs}       
\usepackage{amsfonts}       
\usepackage{nicefrac}       
\usepackage{microtype}      
\usepackage{lipsum}

\usepackage{appendix}
\usepackage[numbers, sort, comma, square]{natbib}
\usepackage{bm}
\usepackage{amsmath}
\usepackage{amssymb}
\usepackage{amsthm}
\usepackage{mathtools}
\usepackage{algorithmic}
\usepackage[ruled]{algorithm2e}
\usepackage{subcaption}
\usepackage{cleveref}

\newtheorem{theorem}{Theorem}
\newtheorem{lemma}{Lemma}

\newtheorem{assumption}{Assumption}
\newtheorem{corollary}{Corollary}
\newtheorem{remark}{Remark}

\crefname{assumption}{assumption}{assumptions}
\crefname{corollary}{corollary}{corollaries}
\crefformat{equation}{Eq.~(#2#1#3)}

\newcommand\norm[1]{\left\lVert#1\right\rVert_2}
\newcommand\inner[2]{\langle#1,#2\rangle}

\title{A New Framework for Variance-Reduced Hamiltonian Monte Carlo}

\author{
  Zhengmian Hu \\
  University of Pittsburgh\\
  Pittsburgh, PA 15213 \\
  \texttt{huzhengmian@gmail.edu} \\
   \And
  Feihu Huang \\
  University of Pittsburgh\\
  Pittsburgh, PA 15213 \\
  \texttt{huangfeihu2018@gmail.com} \\ 
   \And
  Heng Huang \\
  University of Pittsburgh\\
  Pittsburgh, PA 15213 \\
  \texttt{ henghuanghh@gmail.com} \\
}

\begin{document}
\maketitle

\begin{abstract}
We propose a new framework of variance-reduced Hamiltonian Monte Carlo (HMC) methods for sampling from an $L$-smooth and $m$-strongly log-concave distribution, based on a unified formulation of biased and unbiased variance reduction methods.
We study the convergence properties for HMC with gradient estimators which satisfy the Mean-Squared-Error-Bias (MSEB) property.
We show that the unbiased gradient estimators, including SAGA and SVRG, based HMC methods achieve highest gradient efficiency with small batch size under high precision regime, and require $\tilde{O}(N + \kappa^2 d^{\frac{1}{2}} \varepsilon^{-1} + N^{\frac{2}{3}} \kappa^{\frac{4}{3}} d^{\frac{1}{3}} \varepsilon^{-\frac{2}{3}} )$ gradient complexity to achieve $\epsilon$-accuracy in 2-Wasserstein distance. 
Moreover, our HMC methods with biased gradient estimators, such as SARAH and SARGE, require $\tilde{O}(N+\sqrt{N} \kappa^2 d^{\frac{1}{2}} \varepsilon^{-1})$ gradient complexity, which has the same dependency on condition number $\kappa$ and dimension $d$ as full gradient method, but improves the dependency of sample size $N$ for a factor of $N^\frac{1}{2}$.
Experimental results on both synthetic and real-world benchmark data show that our new framework significantly outperforms the full gradient and stochastic gradient HMC approaches.
The earliest version of this paper was submitted to ICML 2020 with three weak accept but was not finally accepted.
\end{abstract}


\section{Introduction}

Markov Chain Monte Carlo (MCMC) algorithms have been widely used for sampling posterior distributions in Bayesian inference. 
Given a dataset $\mathcal{D} = \{ \bm{d}_i \}_{i=1}^n$, we are interested in sampling $p^*(\bm{x}) \propto \exp(-f(\bm{x}))$, where
\begin{equation}
\label{eq:f_prob_explain}
f(\bm{x}) = -\log(p(\bm{x}))-\sum_{i=1}^{n} \log(p(\bm{d}_i|\bm{x})).
\end{equation}
Langevin Monte Carlo (LMC) methods and Hamiltonian Monte Carlo (HMC) methods are two most popular families of gradient-based MCMC.
Langevin Monte Carlo method is based on Langevin dynamics (LD) which is characterized by the following stochastic differential equation (SDE):
\begin{equation}
\label{eq:sde_ld}
d\bm{X}_t = -\nabla f(\bm{X}_t) dt + \sqrt{2} d\bm{B}_t,
\end{equation}
where $\bm{X}_t$ is $d$-dimensional stochastic process, $t\geq0$ denotes time, and $\bm{B}_t$ is the standard $d$-dimensional Brownian motion.
The evolution of probability distribution of  $\bm{X}_t$ can be addressed by the following Fokker-Planck equation:
\begin{equation}
\label{eq:fokker_ld}
\frac{\partial }{\partial t} p_t (\bm{x}) = \nabla^\top (p_t(\bm{x}) \nabla f(\bm{x}) ) + \Delta p_t(\bm{x}).
\end{equation}
When the posterior distribution is well behaved \cite{Chiang1987}, $p_t (\bm{x})$ converges to the unique stationary distribution 
$p^*(\bm{x}) \propto \exp(-f(\bm{x}))$.
One can approximate the Langevin dynamics by applying Euler-Maruyama discretization~\cite{Kloeden2013} on \cref{eq:sde_ld}, and the corresponding update rule is given as:
\begin{equation}
\bm{x}_{k+1}=\bm{x}_{k}-\nabla f(\bm{x}_k) h + \sqrt{2 h} \bm{\epsilon}_k,
\label{eq:step_ld}
\end{equation}
where $\bm{\epsilon}_k$ is a $d$-dimensional standard Gaussian random vector, and $h>0$ is the step size.
\Cref{eq:step_ld} is also referred to as Unadjusted Langevin Algorithm (ULA).
For strongly log-concave and log-smooth posterior distributions,
\cite{Dalalyan2017a,Durmus2016} 
proved that ULA converges to the target density under arbitrary precision in both total variation and 2-Wasserstein distance. 
The non-asymptotic convergence analysis of LMC shows that LMC algorithm can achieve $\varepsilon$ precision in 2-Wasserstein distance after $\tilde{O}(\kappa^2 d / \varepsilon^2)$ iterations~\cite{Dalalyan2017,Dalalyan2019,Durmus2017}. 
If additional Lipschitz continuous condition of the Hessian is satisfied, \cite{Dalalyan2019} showed that the dependency of convergence rate on $\varepsilon$ can be improved to $\tilde{O}(1/\varepsilon)$.
Equivalently, in order to achieve $\varepsilon$ precision in Kullback-Leibler divergence, $\tilde{O}(\kappa^2 d / \varepsilon)$ iterations are required~\cite{Cheng2018}. 

HMC method accelerates the convergence of LMC by Hamiltonian dynamics~\cite{Duane1987,Neal2011}. 
The Hamiltonian dynamics, also known as underdamped Langevin dynamics, can explore the parameter space more efficiently by traversing along contours of a potential energy function, and can be described by the following SDE:
\begin{equation}
\label{eq:sde_hd}
\small
d\bm{X}_t = \xi \bm{V}_t dt, \;
d\bm{V}_t = - \nabla f(\bm{X}_t) dt -\gamma \xi \bm{V}_t dt  + \sqrt{2\gamma} d\bm{B}_t,
\end{equation}
where $\gamma$ is the dissipation parameter, $\xi$ is inverse mass, $\bm{X}_t, \bm{V}_t$ are the $d$-dimensional stochastic processes representing position and momentum.
Under mild condition of posterior distribution, the distribution of ($\bm{X}_t$, $\bm{V}_t$) converges to an unique invariant distribution $p^*(\bm{x},\bm{v}) \propto \exp(-f(\bm{x})-\frac{\xi}{2} \norm{\bm{v}}^2 )$, whose marginal distribution on $\bm{X}_t$ coincides with posterior distribution~\cite{Neal2011}. 
Euler-Maruyama discretization can still be applied to \cref{eq:sde_hd} but that will cancel the accelerated convergence guarantees due to the low-order integration scheme.
One can discretize \cref{eq:sde_hd} by conditioning it on the gradient at $k$-th iteration~\cite{Cheng2018a} as follows:
\begin{equation}
\label{eq:dis_sde_hd}
\small
d\tilde{\bm{V}}_t = - \nabla f(\bm{x}_k) dt -\gamma \xi \tilde{\bm{V}}_t dt  + \sqrt{2\gamma} d\bm{B}_t, \;
d\tilde{\bm{X}}_t = \xi \tilde{\bm{V}}_t dt.
\end{equation}
Integration of the above SDE with a time interval $h$ leads to the update rule of the full gradient HMC algorithm.
Based on a synchronous coupling argument, \cite{Cheng2018a} showed that HMC algorithm can achieve $\varepsilon$ precision in 2-Wasserstein distance after $\tilde{O}(\kappa^2 d^\frac{1}{2} / \varepsilon)$ iterations. 
Under a gradient flow approach, \cite{Ma2019} showed that, with additional Hessian Lipschitz assumption, in order to achieve $\varepsilon$ precision in Kullback-Leibler divergence, $\tilde{O}(\kappa^\frac{3}{2} d^\frac{1}{2} / \varepsilon^\frac{1}{2})$ iterations are required.

The full gradient computation for LMC and HMC could be expensive, especially on large-scale data.
Unbiased stochastic gradient estimator can be used in place of full gradient to bring down the computation requirement for each iteration.
However, stochastic gradient also inevitably introduces extra variance into the sampling algorithm at each step which impedes the convergence.
\cite{Dalalyan2019,Dalalyan2017} studied Stochastic Gradient Langevin Dynamics (SGLD)~\cite{Welling2011} and showed that the gradient complexity of SGLD is $\tilde{O}(\kappa^2 d \sigma^2 / \varepsilon^2)$, where  $\varepsilon$ is accuracy in 2-Wasserstein distance, and $\sigma^2$ is the upper bound of the variance of the stochastic gradient. 
Unlike the full gradient case, assuming extra Hessian smoothness can not improve the dependence of convergence rate on $\varepsilon$ further.
Stochastic Gradient Hamiltonian Monte Carlo (SG-HMC) was studied in \cite{Cheng2018a,Ma2015,Chen2014}. \cite{Cheng2018a} proved the gradient complexity of SG-HMC as $\tilde{O}(\kappa^2 d \sigma^2 / \varepsilon^2)$, which is $\tilde{O}(\frac{d^\frac{1}{2} \sigma^2}{N \varepsilon})$ worse than the full gradient HMC in 2-Wasserstein distance. In both SGLD and SG-HMC, the gradient complexity is dominated by the variance of the stochastic gradient.

Since the potential energy function normally can be decomposed as finite sum of smooth functions as in \cref{eq:f_prob_explain}, variance reduction technique can be employed to reduce the variance of stochastic gradient. 
\citet{Dubey2016} and \citet{Li2019} studied variance reduced LMC and HMC, respectively. They showed that SAGA and SVRG reduce the mean square error (MSE) of the sample path for some test functions, but did not provide gradient complexity with respect to any divergence.
\citet{Baker2019} studied the control-variate technique applied to stochastic gradient Langevin dynamics. Although the convergence rate of control-variate SGLD is no longer dominated by the gradient variance $\sigma^2$, the dependency on $\varepsilon$ is still worse than full gradient method.
\citet{Chatterji2018} studied control-variate underdamped Langevin dynamics (CV-ULD) but their analysis showed that CV-ULD is not guaranteed to converge to arbitrary precision.
With Hessian Lipschitz assumption, \citet{Chatterji2018} proved two sharper convergence rates for SAGA and SVRG based LMC, which recovers the convergence rate of full gradient method under 2-Wasserstein metric in terms of dependence on the sampling accuracy $\varepsilon$.
\citet{Zou2018} analyzed SVRG based HMC with fixed batch size $b=1$, 
however for a fixed step size, the algorithm is not guaranteed to converge after an arbitrary number of steps.

In addition to variance reduction, there are other branches of research that can improve HMC.
Symplectic integration schemes including leapfrog methods leverage symplecticity of canonical transformation and achieve better dependency on $d$ \cite{Mangoubi2018}.
Replica exchange \cite{Chen2019,Deng2020} allows exploring the multi-mode landscape more efficiently. However, these techniques are orthogonal to the research direction of our framework and is of independent interest.

In this paper, we propose a new framework of variance-reduced Hamiltonian Monte Carlo method to leverage most popular variance reduction techniques, including SAGA \cite{Defazio2014}, SVRG \cite{Johnson2013}, SARAH \cite{Nguyen2017}, and SARGE \cite{Driggs2019}. 
Our algorithm was inspired by the recent advance in stochastic optimization \cite{Driggs2019}, which depicts semi-stochastic gradients with so called MSEB property to control the MSE and bias.

To show the advantages of our new methods, we summarize and compare the gradient computational complexity for different Hamiltonian Monte Carlo methods in \Cref{tab:compare}. In \Cref{tab:compare}, $\varepsilon$ represents the accuracy under 2-Wasserstein distance, $N$ is the sample size, $b$ denotes batch size, and all average epoch lengths for SARAH and SVRG are set as $p=O(N/b)$. Our main contributions in this paper can be summarized as follows:
\begin{enumerate}
	\item We propose a new Hamiltonian Monte Carlo framework to leverage popular variance reduction techniques, including both biased and unbiased gradient estimators. 
	\item In theoretical analysis, we prove the convergence of our framework with MSEB estimator in a general manner. As a specialization, we consider four variance-reduced gradient estimators, SAGA, SVRG, SARAH, and SARGE, and derive the convergence rate under 2-Wasserstein metric for them. All variance reduction methods considered in this paper enjoy better convergence rate than existing full gradient method and stochastic gradient methods. 
	\item To the best of our knowledge, the biased variance reduction techniques, including SARAH and SARGE, have not been incorporated into stochastic HMC for sampling strongly-log-concave distribution, and this paper provides the first convergence result for them.
\end{enumerate}

\begin{table*}[t]
	\centering
	\begin{small}
		
		\begin{tabular}{lcccc}
			\toprule
			Methods & Reference & Batch size & Gradient complexity & Converge at Infinite Time \\
			\midrule
			HMC & \cite{Cheng2018a} & $N$ & $\tilde{O}(N \kappa^2 d^\frac{1}{2} / \varepsilon)$ &Y\\ 
			SG-HMC & \cite{Cheng2018a} & $O(1)$ &  $\tilde{O}(\kappa^2 \sigma^2 d / \varepsilon^2)$ &Y\\ 
			SVRG-HMC & \cite{Zou2018} & $1$ & $\tilde{O}(N+\kappa^2 d^\frac{1}{2}/\varepsilon + N^\frac{2}{3} \kappa^\frac{3}{4} d^\frac{1}{3}/\varepsilon^\frac{2}{3})$ &N\\ 
			SVRG-HMC & Ours & $1$ & $\tilde{O}(N \kappa^2+\kappa^2 d^\frac{1}{2}/\varepsilon + N^\frac{2}{3} \kappa^\frac{3}{4} d^\frac{1}{3}/\varepsilon^\frac{2}{3})$ &Y\\ 
			SAGA-HMC & Ours & $1$ & $\tilde{O}(N \kappa^2+\kappa^2 d^\frac{1}{2}/\varepsilon + N^\frac{2}{3} \kappa^\frac{3}{4} d^\frac{1}{3}/\varepsilon^\frac{2}{3})$ &Y\\ 
			SVRG-HMC & Ours & $b$ & $\tilde{O}(N+N \kappa^2/b^\frac{1}{2}+b \kappa^2 d^\frac{1}{2}/\varepsilon + N^\frac{2}{3} \kappa^\frac{3}{4} d^\frac{1}{3}/\varepsilon^\frac{2}{3})$ &Y\\ 
			SAGA-HMC & Ours& $b$ & $\tilde{O}(N+N \kappa^2/b^\frac{1}{2}+b \kappa^2 d^\frac{1}{2}/\varepsilon + N^\frac{2}{3} \kappa^\frac{3}{4} d^\frac{1}{3}/\varepsilon^\frac{2}{3})$ &Y\\ 
			SARAH-HMC & Ours & $1$ & $\tilde{O}(N+N^\frac{1}{2} \kappa^2 d^\frac{1}{2}/\varepsilon)$ &Y\\ 
			SARGE-HMC & Ours & $1$ & $\tilde{O}(N+N^\frac{1}{2} \kappa^2 d^\frac{1}{2}/\varepsilon)$ &Y\\ 
			\bottomrule
		\end{tabular}
		
	\end{small}
	
	\caption{Gradient complexity of different Hamiltonian Monte Carlo methods for sampling $L$-smooth and $m$-strongly log-concave distribution. We accept the large mini-batch size $b>1$.}
	\label{tab:compare}
\end{table*}

\section{Preliminary}
In order to show the convergence of our variance reduced HMC framework for sampling from an  $L$-smooth and $m$-strongly log-concave distribution $p^\ast \propto e^{-f(x)}$, we need to introduce some mild assumptions on the potential energy function $f(x): \mathbb{R}^d\rightarrow\mathbb{R}$ as follows:

\begin{assumption}[Sum-decomposable]\label{as:1}
	$f(\bm{x}) = \sum_{i=1}^{N} f_i(\bm{x})$, where integer $N$ is the sample size.
\end{assumption}
\begin{assumption}[Smoothness]\label{as:2}
	Each function $f_i$ is continuously-differentiable on $\mathbb{R}^d$ and there exists a constant $\tilde{L} > 0$, such that 
	$$\norm{\nabla f_i(\bm{x}) - \nabla f_i(\bm{y})} \leq \tilde{L} \norm{\bm{x}-\bm{y}}$$
	for any $\bm{x}, \bm{y} \in \mathbb{R}^d$. It can be easily verified that $ f(\bm{x})$ is $L$-smooth with $L = N \tilde{L}$.
\end{assumption}
\begin{assumption}[Strong Convexity]\label{as:3}
	There exists a constant $m > 0$ such that 
	$$f(\bm{x}) - f(\bm{y}) \geq \inner{\nabla f(\bm{y})}{\bm{x} - \bm{y}} + \frac{m}{2} \norm{\bm{x} - \bm{y}}^2.$$
	We define the condition number $\kappa \coloneqq L/m$.
\end{assumption}
\begin{assumption}[Optimal at Zero]\label{as:4}
	Without loss of generality, we assume $\bm{x}^\ast=0$ and $f(\bm{x}^\ast)=0$ where $\bm{x}^\ast$ is the global minimizer for the strongly convex potential energy function.
\end{assumption}
\noindent\textbf{Wasserstein Distance}: Given a pair of probability measures $\mu$ and $\nu$, we define a transference plan $\zeta$ between $\mu$ and $\nu$ as a joint distributions such that marginal distribution of the first set of coordinates is $\mu$ and marginal distribution of the second set of coordinates is $\nu$. We denote $\Gamma(\mu,\nu)$ as the set of all transference plans. We define the 2-Wasserstein distance between $\mu$ and $\nu$ as follows,	
$$W_2^2(\mu,\nu) = \inf_{\zeta \in \Gamma(\mu,\nu)} \int \norm{\bm{x} - \bm{y}}^2 d\zeta(x,y). $$
\textbf{MSEB property}: 
Given a parameter sequence $\{\bm{x}_k\}$ and a function $f$,
a stochastic gradient estimator $\tilde{\nabla}$ is a series of vectors $\tilde{\nabla}_k$ generated from $\{\bm{x}_{i}\}_{i=0}^{k}$. 
We say that a stochastic gradient estimator $\tilde{\nabla}$ satisfies MSEB property if there exist constants $M_1,M_2\geq 0, \rho_M,\rho_B,\rho_F \in ( 0,1 ]$ and sequences $\mathcal{M}_k$ and $\mathcal{F}_k$ such that 
\begin{eqnarray}
\label{eq:mseb}
&\nabla f(\bm(x_{k+1})) - \mathbb{E}_k \tilde{\nabla}_{k+1} = (1-\rho_B) (\nabla f(\bm(x_k)) - \tilde{\nabla}_k) \nonumber\\
&\mathbb{E} \norm{\tilde{\nabla}_{k+1} - \nabla f(\bm{x}_{k+1})}^2 \leq \mathcal{M}_k \nonumber\\
&\mathcal{M}_k \leq M_1 Q_k+\mathcal{F}_k\!+\!(1\!-\!\rho_M ) \mathcal{M}_{k-1} \\
&\mathcal{F}_k \leq \sum_{l=0}^{k} M_2 (1-\rho_F)^{k-l} Q_l \nonumber\\
&Q_k = N \sum_{i=1}^{N} \mathbb{E} \norm{\nabla f_i (\bm{x}_{k+1}) - \nabla f_i (\bm{x}_{k})}^2. \nonumber
\end{eqnarray}
$\mathbb{E}_k$ means expectation conditioned on all variables at $k$-th step and all previous steps. 
MSEB property controls the bias and MSE of the gradient estimator with a weighted sum of gradient changes 
$\norm{\nabla f_i(\bm(x_{k+1}))-\nabla f_i(\bm(x_k))}^2$ along the previous sample path.
Note that many popular gradient estimators including SAGA \cite{Defazio2014}, SVRG \cite{Johnson2013}, SARAH \cite{Nguyen2017}, and SARGE \cite{Driggs2019} satisfy MSEB property.

\section{A New Framework for Variance-Reduced Hamiltonian Monte Carlo}

In the section, we propose a new framework for variance-reduced Hamiltonian Monte Carlo based on the MSEB property. 

We first derive the update rule by integrating the SDE of Hamiltonian dynamics \cref{eq:dis_sde_hd}.
With step as $h$, we obtain the following update rule:
\begin{equation}
\label{eq:step_vrhd}
\begin{aligned}
\bm{x}_{k+1} = \tilde{\bm{X}}_h =  &
\bm{x}_k +  \frac{1}{\gamma} (1-e^{-\gamma \xi h}) \bm{v}_k  \\
& - \frac{1}{\gamma} (h- \frac{1}{\gamma \xi } (1-e^{-\gamma \xi h}) ) \nabla f(\bm{x_k})
+ \bm{e}^x_k
, \\
\bm{v}_{k+1} = \tilde{\bm{V}}_h =  & 
e^{-\gamma \xi h}
\bm{v}_k - \frac{1}{\gamma \xi } (1-e^{-\gamma \xi h}) \nabla f(\bm{x_k}) + \bm{e}^v_k
,
\end{aligned}
\end{equation}
where $\bm{e}^v_k$ and $\bm{e}^v_k$ denote Gaussian random vectors with zero mean and the following covariance:
\begin{eqnarray}
\label{eq:covar_step_vrhd}
&	\mathbb{E} (\bm{e}^v_k {\bm{e}^v_k}^\top) & = \frac{1}{\xi} (1-e^{-2 \gamma \xi h}) \bm{I}_{d\times d}
\nonumber\\
&	\mathbb{E} (\bm{e}^x_k {\bm{e}^v_k}^\top) & = \frac{1}{\gamma \xi} (1+e^{-2\gamma \xi h} - 2 e^{-\gamma \xi h}) \bm{I}_{d\times d}
\\
&	\mathbb{E} (\bm{e}^x_k {\bm{e}^x_k}^\top) & = \frac{1}{\gamma^2 \xi} (
2 \gamma \xi h - 3 + 4 e^{-\gamma \xi h}-e^{-2 \gamma \xi h}
) \bm{I}_{d\times d} \nonumber
\end{eqnarray}
For the sum decomposable function $f(\bm{x}) = \sum_{i=1}^{N} f_i(\bm{x})$, the stochastic gradient can be used to reduce the computation for single iteration by substituting full gradient $\nabla f(\bm{x_k})$ with stochastic gradient
$\frac{N}{|\mathcal{B}_k|} \sum_{i\in \mathcal{B}_k} \nabla f_i(\bm{x_k})$. 
However the gradient error of stochastic gradient can be large and hinders the convergence.
Variance reduction techniques could remedy this problem by using historical gradient information to reduce the gradient error of current iterate.
The idea of variance reduction has been widely used in optimization and there exist many popular choices for variance reduction techniques such as SAGA, SVRG, SARGE and SARAH.

In order to leverage the advances of different variance reduction methods to accelerate HMC, 
we use MSEB property to deal with different variance reduction methods uniformly, and propose a framework that is compatible with all MSEB gradient estimators.
Our framework is summarized in \Cref{alg:vrld}.
\begin{algorithm}[tb]
	\caption{Variance-Reduced HMC (VR-HMC) Algorithm}
	\label{alg:vrld}
	\begin{algorithmic}
		\STATE {\bfseries Input:} Initial point $(\bm{x}_0, \bm{v}_0)$, smoothness parameter $L$ and step size $h>0$.\\
		\FOR {$k=0$ {\bfseries to} $K-1$}
		\STATE Generate the variance reduced stochastic gradient $\tilde{\nabla}_k$ which satisfied MSEB property;\\
		\STATE Generate Gaussian random vectors  $\bm{e}^x_k$ and $\bm{e}^v_k$ based with covariance in \eqref{eq:covar_step_vrhd};\\   
		\STATE Update $\bm{x}_{k+1} =  \bm{x}_k +  \frac{1}{\gamma} (1-e^{-\gamma \xi h}) \bm{v}_k  - \frac{1}{\gamma} (h- \frac{1}{\gamma \xi } (1-e^{-\gamma \xi h}) ) \tilde{\nabla}_k + \bm{e}^x_k$;\\
		\STATE Update $\bm{v}_{k+1} = 
		e^{-\gamma \xi h}\bm{v}_k - \frac{1}{\gamma \xi } (1-e^{-\gamma \xi h}) \tilde{\nabla}_k + \bm{e}^v_k$.\\
		\ENDFOR
		\STATE {\bfseries Output:}  $\bm{v}_K$.
	\end{algorithmic}
\end{algorithm}
Now we can state the convergence guarantee for our variance reduced HMC framework.
\begin{theorem}\label{th:main}
	Let $f$ be a function satisfying \Cref{as:1,as:2,as:3,as:4}, $\tilde{\nabla}_k$ is an MSEB estimator. 
	Let the initial point be $(\bm{x}_0,0)$ and the initial distribution be $p_0(\bm{x},\bm{v}) = \delta_{\bm{x} = \bm{x}_0} \delta_{\bm{v} = \bm{0}} $. 
	With small enough step size $h$ satisfying 
	$L h \leq \frac{1}{10 \kappa} \min(1,\frac{1}{\sqrt{\Theta}})$,
	denoting $q_k = (\bm{x}_k,\bm{x}_k+\bm{v}_k)$, 
	after running the \Cref{alg:vrld} for $k$ iterations, we have:
	\begin{equation*}
	\begin{split}
	W_2(q_k, q^*) \leq & e^{-\frac{k h m}{2}} W_2(q_0, q^*) + 8 \sqrt{L F_2} \kappa h 
	\\
	& + 4 \sqrt{\Theta F_1} \left(2(1-\rho_B)\sqrt{L} \kappa h + L \sqrt{\kappa} h^\frac{3}{2}\right)
	,
	\end{split}
	\end{equation*}
	where $p^*(\bm{x},\bm{v}) = q^*(\bm{x},\bm{x}+\bm{v}) \propto \exp(-f(\bm{x})-\frac{\xi}{2} \norm{\bm{v}}^2 )$, $\Theta= \frac{M_1}{\rho_M} + \frac{M_2}{\rho_M \rho_F}$, $F_1 = 13 L \norm{\bm{x}_0}^2 + 24 \kappa d$ and 
	$F_2 = 97 L \norm{\bm{x}_0}^2 + 181 \kappa d$.
\end{theorem}
\begin{corollary}\label{co:unbiased}
	For unbiased gradient estimator, we have $\rho_B = 1$.
	Under the same conditions as in \Cref{th:main}, 
	let the step size be 
	$$L h \leq \min(\varepsilon \kappa^{-\frac{3}{2}} L^{\frac{1}{2}} d^{-\frac{1}{2}}, \varepsilon^\frac{2}{3} \kappa^{-\frac{2}{3}} L^\frac{1}{3} d^{-\frac{1}{3}}, \frac{1}{10 \kappa \max(1, \sqrt{\Theta})}) .$$
	The output of \Cref{alg:vrld} with unbiased gradient estimator satisfies $W_2(q_k, q^*) \leq \varepsilon$,
	within 
	$\tilde{O} (\sqrt{\Theta} \kappa^2 + \kappa^2 d^\frac{1}{2} \varepsilon^{-1} + \Theta^\frac{1}{3} \kappa^\frac{4}{3} d^\frac{1}{3} \varepsilon^{-\frac{2}{3}})$
	iterations.
\end{corollary}
\begin{remark}\label{re:smallstep}
	The first term $\sqrt{\Theta} \kappa^2$ in the iteration complexity comes from the restriction of small step size $Lh \leq \frac{1}{10 \kappa \sqrt{\Theta}}$ and is independent of precision $\varepsilon$. 
	If we assume high precision condition
	$\varepsilon \leq \frac{d^\frac{1}{2}}{\min(\sqrt{\Theta}, \kappa \Theta^\frac{1}{4})}$
	, the first term is dominated by the second or third term thus the iteration complexity would be $\tilde{O} (\kappa^2 d^\frac{1}{2} \varepsilon^{-1} + \Theta^\frac{1}{3} \kappa^\frac{4}{3} d^\frac{1}{3} \varepsilon^{-\frac{2}{3}})$. 
	The restriction of small step size is necessary for the convergence after running the algorithm for arbitrary long time. 
	We notice that \cite{Zou2018} didn't assume small step size. As a result, they can only guarantee the convergence for $k < O(\frac{1}{L^2 h^2 \kappa})$.
	
\end{remark}
\begin{corollary}\label{co:biased}
	For biased gradient estimator, we have $\rho_B < 1$.
	Under the same conditions as in \Cref{th:main}, 
	for precision $\varepsilon>0$,
	let the step size satisfy
	$$L h \leq \varepsilon \kappa^{-\frac{3}{2}} L^{\frac{1}{2}} d^{-\frac{1}{2}}\min(1, \frac{1}{\sqrt{\Theta}}) .$$
	The output distribution of \Cref{alg:vrld} with biased gradient estimator satisfies $W_2(q_k, q^*) \leq \varepsilon$,
	within 	$\tilde{O} ((1+\sqrt{\Theta})\kappa^2 d^\frac{1}{2} \varepsilon^{-1} )$
	iterations.
\end{corollary}
\begin{remark}
	Recall that the iteration complexity of SG-HMC is $\tilde{O}(\kappa^2 d \sigma^2 / \varepsilon^2)$ \cite{Cheng2018a}.
	Compared to SG-HMC, both biased and unbiased variance-reduced HMC improve the dependency of $\varepsilon$.
	Compared to the convergence rate $\tilde{O}(\kappa^2 d^\frac{1}{2}/\varepsilon)$ of full gradient HMC \cite{Cheng2018a}, 
	our algorithm with unbiased gradient estimator is penalized by a term $\Theta^\frac{1}{3} \kappa^\frac{4}{3} d^\frac{1}{3} \varepsilon^{-\frac{2}{3}}$, and our algorithm with biased gradient estimator is penalized by a factor of $O(1+\sqrt{\Theta})$.
	Therefore, our methods with MSEB gradient estimator takes more iterations than full gradient HMC to achieve same accuracy.
	This regression comes from the perturbation of inaccurate gradient estimator and is controlled by parameter $\Theta$.
\end{remark}

\subsection{Convergence Properties for Specific Estimators}
Under \Cref{th:main}, we can prove the convergence rate of a specific gradient estimator for \Cref{alg:vrld} by just establishing bounds on the MSEB terms in \eqref{eq:mseb}.

We first consider full gradient as a special case of MSEB gradient estimator where no bias or mean square error exists.
\begin{corollary}[Full Gradient]\label{co:fg}
	When we use full gradient in \Cref{alg:vrld}, $\Theta = 0$, 
	we need $\tilde{O}( \kappa^2 d^\frac{1}{2}/\varepsilon)$ iterations to achieve $\varepsilon$ accuracy in 2-Wasserstein distance. Given that computing a full gradient requires $N$ queries on the gradient of each component function $f_i(\bm{x})$, we can show the gradient complexity is $\tilde{O}( N \kappa^2 d^\frac{1}{2}/\varepsilon)$.
\end{corollary}
\begin{remark}
	Recall that previous research \cite{Cheng2018a} has shown that the gradient complexity of full gradient HMC is $\tilde{O}(N \kappa^2 d^\frac{1}{2}/\varepsilon)$.
	Our result can successfully achieve such gradient complexity, which implies that our analysis is tight under the notation of MSEB estimator.
\end{remark}
Next we combine unbiased variance reduction methods with our framework to improve the gradient complexity.
We choose two most popular unbiased variance reduction methods SVRG and SAGA.

SVRG was first proposed for strongly convex optimization in \cite{Johnson2013} as an unbiased variance reduction technique to accelerate the convergence to the global minimizer. The estimated gradient is calculated in the following way where $B_k$ is the batch of $k$-th iteration:
\begin{align*}
\tilde{\nabla}_k^{SVRG} = & \frac{N}{b} \sum_{i\in B_k} \left( \nabla f_i(\bm{x}_k) - \nabla f_i(\tilde{\bm{x}}) \right) + \nabla f(\tilde{\bm{x}})\,.
\end{align*} 

SVRG updates the snapshot $\tilde{\bm{x}}$ periodically and computes the full gradient $\nabla f(\tilde{\bm{x}})$ on the snapshot. 
Despite extra gradient queries,
it was shown that SVRG has lower gradient MSE and enjoys better gradient complexity under many setting of optimization.

The original SVRG has an inner and outer loop structure which is not compatible with our framework.
In order to combine it with MSEB property, we consider a variant of SVRG where the snapshot is updated with probability 
$\frac{1}{p}$ at each iteration, such that the average interval between snapshot updates is $p$ iterations.

SAGA \cite{Defazio2014} is another popular variance-reduced algorithm. Instead of calculating the full gradient of a previous snapshot, the most recent gradient information $\bm{\phi}_k^i$ of each component function $f_i(\bm{x})$ is stored. SAGA estimates the gradient as follows:
\begin{equation*}
\tilde{\nabla}_k^{SAGA} = 
\frac{N}{b} \sum_{i\in B_k} \left( \bm{\phi}_k^i - \bm{\phi}_{k-1}^i \right) + \sum_{i=1}^N \bm{\phi}_{k-1}^i\,.
\end{equation*}

The most recent gradient $\bm{\phi}_k^i$ is set as $\nabla f_i(\bm{x}_k)$ if the component functions $f_i$ is in the batch of $k$-th iteration, otherwise it remains the same as $\bm{\phi}_{k-1}^i$.
SAGA avoids the extra gradient computation compared with SVRG, however, it requires much more memory to store the old gradient information for each data point.

SAGA and SVRG are both unbiased gradient estimators since $\mathbb{E}_k \tilde{\nabla}_{k+1} = \nabla f(\bm{x}_{k+1})$. 
According to \Cref{co:unbiased}, we can obtain the gradient complexity by just studying the MSEB terms. 
\begin{corollary}[SVRG]\label{co:svrg}
	When SVRG is used in \Cref{alg:vrld}, 
	let $b$ be the batch size, $p$ be the average number of iterations between snapshot updates,
	and we have $\Theta = \frac{6p^2}{b}$,
	and for each iteration, $N/p+2 b$ gradient queries are needed in average.
	The gradient complexity is 
	$\tilde{O}(N+ 
	(N/b^\frac{1}{2}+p b^\frac{1}{2}) \kappa^2+
	b \kappa^2 d^\frac{1}{2}/\varepsilon+
	(N/(p b)^\frac{1}{3} + (p b)^\frac{2}{3})\kappa^\frac{4}{3} d^\frac{1}{3}/\varepsilon^\frac{2}{3})$.
	Most of the time, we choose $p=O(N/b)$,
	then the gradient complexity is $\tilde{O}(N+ 
	N \kappa^2/b^\frac{1}{2}+
	b \kappa^2 d^\frac{1}{2}/\varepsilon+
	N^\frac{2}{3}\kappa^\frac{4}{3} d^\frac{1}{3}/\varepsilon^\frac{2}{3})$.
	Let the batch size be $b=1$, 
	the gradient complexity is $\tilde{O}(N \kappa^2+\kappa^2 d^\frac{1}{2}/\varepsilon+N^\frac{2}{3}\kappa^\frac{4}{3} d^\frac{1}{3}/\varepsilon^\frac{2}{3})$.
\end{corollary}
\begin{corollary}[SAGA]\label{co:saga}
	When SAGA is used in \Cref{alg:vrld}, 
	let $b$ be the batch size, 
	and we have $\Theta = \frac{6 N^2}{b^3}$, 
	and the gradient complexity is $\tilde{O}(N+N \kappa^2/b^\frac{1}{2}+b \kappa^2 d^\frac{1}{2}/\varepsilon+N^\frac{2}{3}\kappa^\frac{4}{3} d^\frac{1}{3}/\varepsilon^\frac{2}{3})$.
	Let the batch size be $b=1$, 
	the gradient complexity is $\tilde{O}(N\kappa^2+\kappa^2 d^\frac{1}{2}/\varepsilon+N^\frac{2}{3}\kappa^\frac{4}{3} d^\frac{1}{3}/\varepsilon^\frac{2}{3})$.
\end{corollary}
If we set $p=N/b$ for SVRG, both SAGA and SVRG have the same $\Theta$, which means they have similar effect on reducing the variance of the gradient estimation.
As a result, our HMC framework with these two techniques have same gradient complexity 
$\tilde{O}(N+
N \kappa^2/b^\frac{1}{2}+
b \kappa^2 d^\frac{1}{2}/\varepsilon+
N^\frac{2}{3}\kappa^\frac{4}{3} d^\frac{1}{3}/\varepsilon^\frac{2}{3})$.

Each term in the above gradient complexity is strictly better than the gradient complexity $\tilde{O} (N \kappa^2 d^\frac{1}{2} /\varepsilon)$ of full gradient method.
Compared with the gradient complexity of stochastic gradient HMC $\tilde{O} (\kappa^2 d \sigma^2 /\varepsilon^2)$, where we assume $\mathbb{E} \norm{\nabla f_i(\bm{x})- \nabla f(\bm{x})}^2\leq \sigma^2$,
\emph{our result has better dependency on $d$ and $\epsilon$, and our analysis doesn't depend on extra assumption on the bounded variance of stochastic gradient.}

\begin{table*}[t]
	\centering
	\begin{small}
		\begin{tabular}{lcccccc}
			\toprule
			Methods &  HMC & SG-HMC & SVRG-HMC & SARAH-HMC & SAGA-HMC & SARGE-HMC  \\
			\midrule
			Potential MSE ($\times 10^{-5}$) & 
			$19 \pm 3$ & 
			$1187 \pm 30$ & 
			$19 \pm 3$ & 
			$19 \pm 3$ &
			$21 \pm 3$ &
			$356 \pm 17$
			\\
			Gradient MSE & 
			$0.0$ & 
			$845.549 \pm 0.022$ & 
			$0.0$ &
			$0.0$ & 
			$21.146 \pm 0.005$ &
			$1.1042 \pm 0.0003$
			\\
			\bottomrule
		\end{tabular}
	\end{small}
	\caption{Potential energy MSE and gradient MSE for different Hamiltonian Monte Carlo Methods on synthetic data.}
	\label{tab:syn_mse}
\end{table*}

We further discuss the choice strategy of the batch size under different regimes: 
\begin{enumerate}
	\item 
	Under low precision regime, $\varepsilon \geq \frac{d^\frac{1}{2}}{\min(\sqrt{\Theta}, \kappa \Theta^\frac{1}{4})} = 
	\max(\frac{d^\frac{1}{2} b^\frac{3}{2}}{N}, \frac{d^\frac{1}{2} b^\frac{3}{4}}{N^\frac{1}{2} \kappa})$,
	the last two terms that are $\varepsilon$ dependent are dominated by the second term. 
	Therefore the gradient complexity is $\tilde{O}(N+N \kappa^2/b^\frac{1}{2})$. 
	Clearly in this regime, increasing the batch size could help  decrease the gradient complexity.
	\item 
	Under high precision regime, $\varepsilon \leq \max(\frac{d^\frac{1}{2} b^\frac{3}{2}}{N}, \frac{d^\frac{1}{2} b^\frac{3}{4}}{N^\frac{1}{2} \kappa})$,
	the gradient complexity changes to $\tilde{O}(N+b \kappa^2 d^\frac{1}{2}/\varepsilon+N^\frac{2}{3}\kappa^\frac{4}{3} d^\frac{1}{3}/\varepsilon^\frac{2}{3})$.
	This result encourage us to decrease the batch size when the term $b \kappa^2 d^\frac{1}{2}/\varepsilon$ is dominant.
	\item
	Under high precision regime, if we further assume $\varepsilon \geq \frac{b^3 \kappa^2 d}{N^2}$, 
	then the last term dominates the second and third terms. 
	The gradient complexity changes to $\tilde{O}(N+N^\frac{2}{3}\kappa^\frac{4}{3} d^\frac{1}{3}/\varepsilon^\frac{2}{3})$ and is independent of batch size $b$.
	Therefore, we can increase the batch to as large as $\frac{\varepsilon^\frac{1}{3} N^\frac{2}{3}}{\kappa^\frac{2}{3} d^\frac{1}{6}}$
	without hurting the convergence rate.
	\item
	Under high precision regime, if we assume $\varepsilon \leq \frac{b^3 \kappa^2 d}{N^2}$, 
	the gradient complexity changes to $\tilde{O}(N+b \kappa^2 d^\frac{1}{2}/\varepsilon)$,
	which is positively correlated with batch size $b$.
	If 
	$\varepsilon \leq \min(\frac{\kappa^2 d}{N^2},\frac{d^\frac{1}{2}}{N})$
	the best gradient complexity is achieved by setting $b=1$.
\end{enumerate}


Biased stochastic gradient methods were not wildly adopted in previous sampling methods because of the difficulties in the algorithm convergence guarantee and theoretical analysis. 
We show that the biased estimators can still be applied together with our HMC framework for sampling strongly-log-concave distribution to achieve acceleration.
However, the bias might outweigh the benefits of a lower gradient MSE and hurt the convergence rate.
In this paper, we consider SARAH and SARGE because they can further reduce the MSE of the gradient estimation.

SARAH \cite{Nguyen2017} is very similar to SVRG
but estimates the full gradient in a recursive way:
\begin{equation*}
\tilde{\nabla}_k^{SARAH} = 
\frac{N}{b} \sum_{i\in B_k} \left(
\nabla f_i(\bm{x}_k)-\nabla f_i(\bm{x}_{k-1})
\right)+\tilde{\nabla}_{k-1}^{SARAH}
\end{equation*}
SARAH also needs to reset gradient estimator $\tilde{\nabla}_k^{SARAH}$ to full gradient $\nabla f(\bm{x}_k)$ periodically, which leads to inner and outer loop structure in the algorithm.
In order to prove the MSEB property for SARAH, we consider a variant of SARAH where the full gradient is calculated with probability $\frac{1}{p}$ at each iteration.

SARGE \cite{Driggs2019} doesn't require computing the full gradient repeatedly but requires the extra storage.
The gradient is estimated as follows:
\begin{equation*}
\small
\tilde{\nabla}_k^{SARGE} =  \frac{N}{b}\sum_{i\in B_k} 
\left(\bm{\psi_k^i} -\bm{\psi_{k-1}^i}\right)
+ \sum_{i=1}^{N} \bm{\psi_{k-1}^i}
+(1-\frac{b}{N}) \tilde{\nabla}_{k-1}^{SARGE}
\end{equation*}
where  $\bm{\psi_k^i}$ is updated as $\nabla f_i(\bm{x}_k)-(1-\frac{b}{N})\nabla f_i(\bm{x}_{k-1})$ if $i$ is in the batch, otherwise remains the same.

We then deduce the convergence guarantee for our framework based on SARAH and SARGE.
\begin{corollary}[SARAH]\label{co:sarah}
	When using SARAH in \Cref{alg:vrld}, 
	let $b$ be the batch size, $p$ be the average number of iterations between calculating full gradient,
	we have $\Theta = p$  and $\rho_B  = \frac{1}{p}$, , 
	and the gradient complexity is $\tilde{O}(N+(b+b p^\frac{1}{2})\kappa^2 d^\frac{1}{2}/\varepsilon)$.
	Let the batch size be $b=1$, and the average interval between full gradient updates be $p=O(N/b)$,
	the gradient complexity is $\tilde{O}(N+N^\frac{1}{2}\kappa^2 d^\frac{1}{2}/\varepsilon)$.
\end{corollary}
\begin{corollary}[SARGE]\label{co:sarge}
	When using SVRG in \Cref{alg:vrld}, 
	let $b$ be the batch size,
	we have $\Theta = \frac{72 N}{b}+\frac{108 N}{b^2}$ and $\rho_B  = \frac{b}{N}$, 
	and the gradient complexity is $\tilde{O}(N+(b+N^\frac{1}{2} b^\frac{1}{2})\kappa^2 d^\frac{1}{2}/\varepsilon)$.
	Let the batch size be $b=1$,
	the gradient complexity is $\tilde{O}(N+N^\frac{1}{2}\kappa^2 d^\frac{1}{2}/\varepsilon)$.
\end{corollary}
Both SARAH and SARGE achieve their best gradient complexity of $\tilde{O}(N+N^\frac{1}{2}\kappa^2 d^\frac{1}{2}/\varepsilon)$ with small batch $b=1$.
Compared with full gradient methods, the dependency of dataset size $N$ is improved by a factor $N^\frac{1}{2}$.
If compared with stochastic gradient methods, the dependency of $\varepsilon$ is improved by a factor of $1/\varepsilon$.

Compared with SAGA and SVRG, SARAH and SARGE have much smaller gradient MSE since $\Theta$ has better dependency of $N$.
However, this comes with the price of non-zero gradient bias, which hurts the convergence rate in dependency of $\epsilon$.
Therefore, in the high precision regime, 
HMC with biased gradient estimator could converge slower than HMC with unbiased gradient estimators even if with smaller gradient MSE.

\begin{figure*}[t]
	\centering
	\begin{subfigure}[b]{0.24\textwidth}
		\centering
		\includegraphics[width=\textwidth]{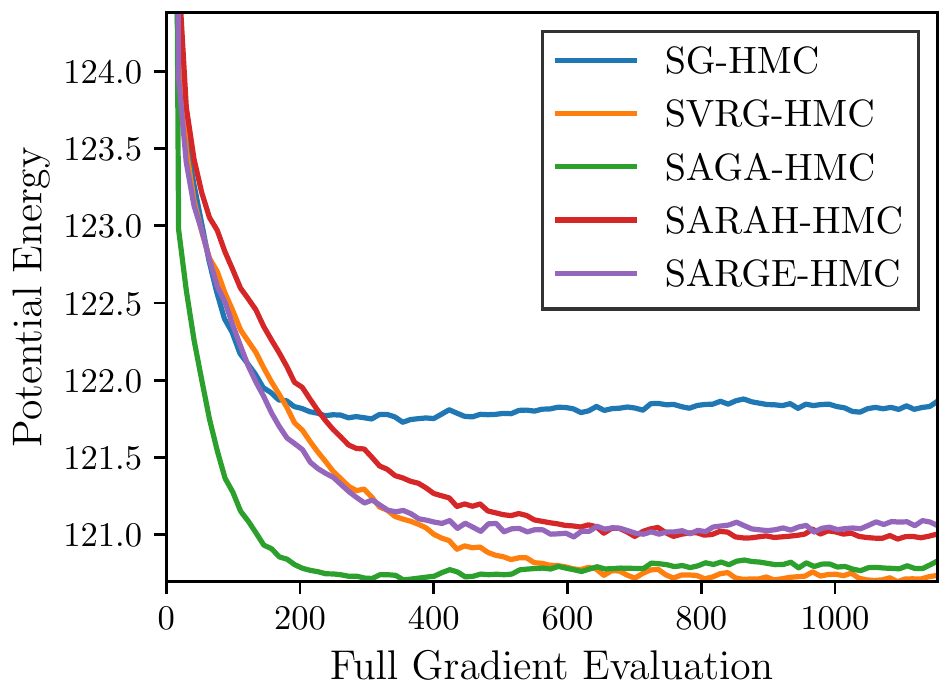}
		\caption{australian}
	\end{subfigure}
	\begin{subfigure}[b]{0.24\textwidth}  
		\centering
		\includegraphics[width=\textwidth]{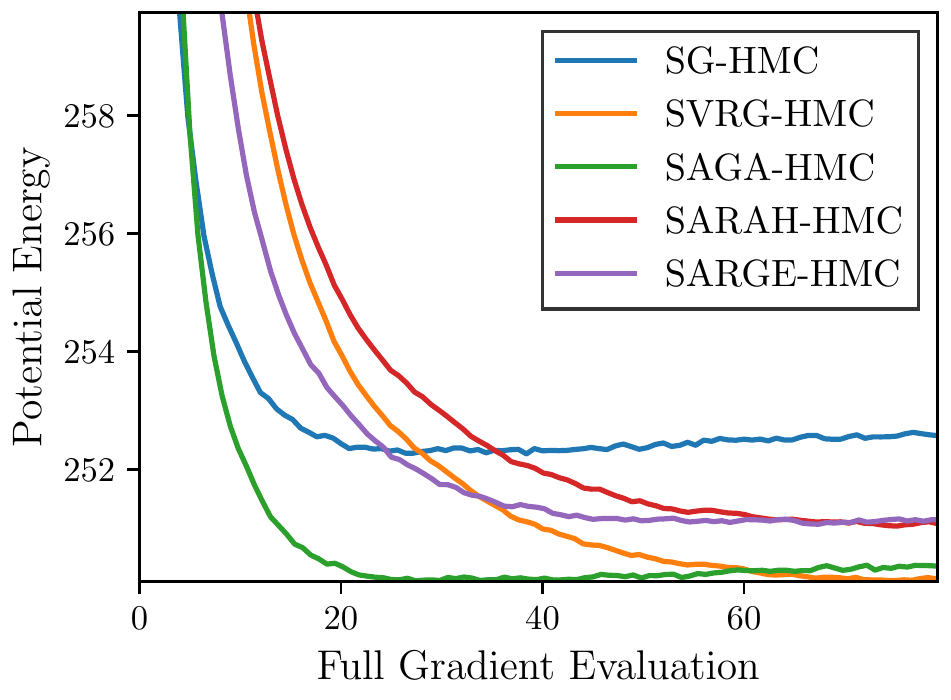}
		\caption{german}
	\end{subfigure}
	\begin{subfigure}[b]{0.24\textwidth}
		\centering
		\includegraphics[width=\textwidth]{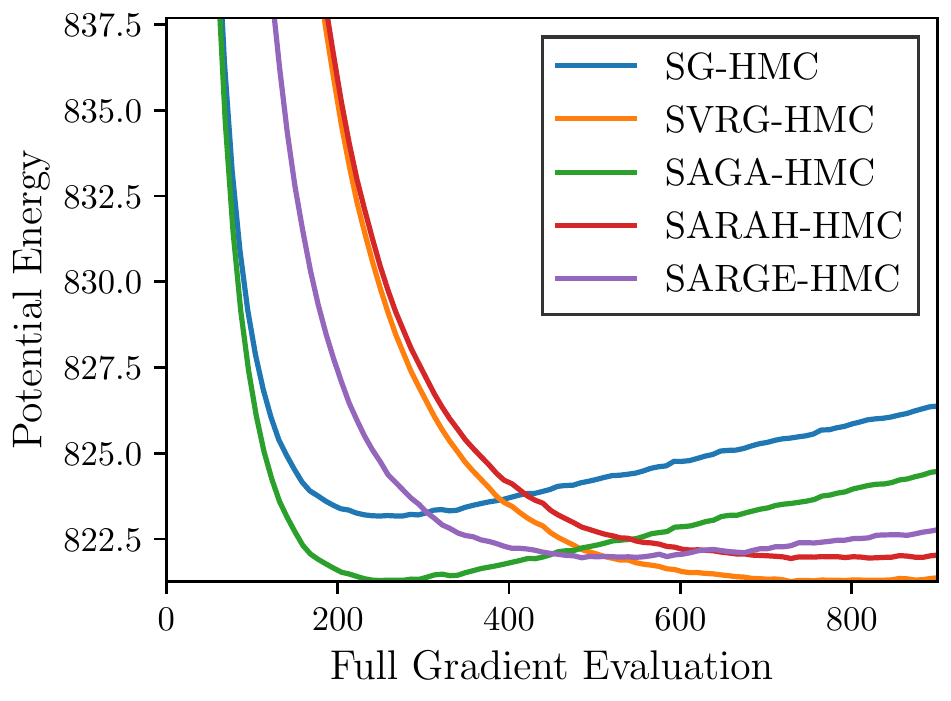}
		\caption{phishing}
	\end{subfigure}
	\begin{subfigure}[b]{0.24\textwidth}
		\centering
		\includegraphics[width=\textwidth]{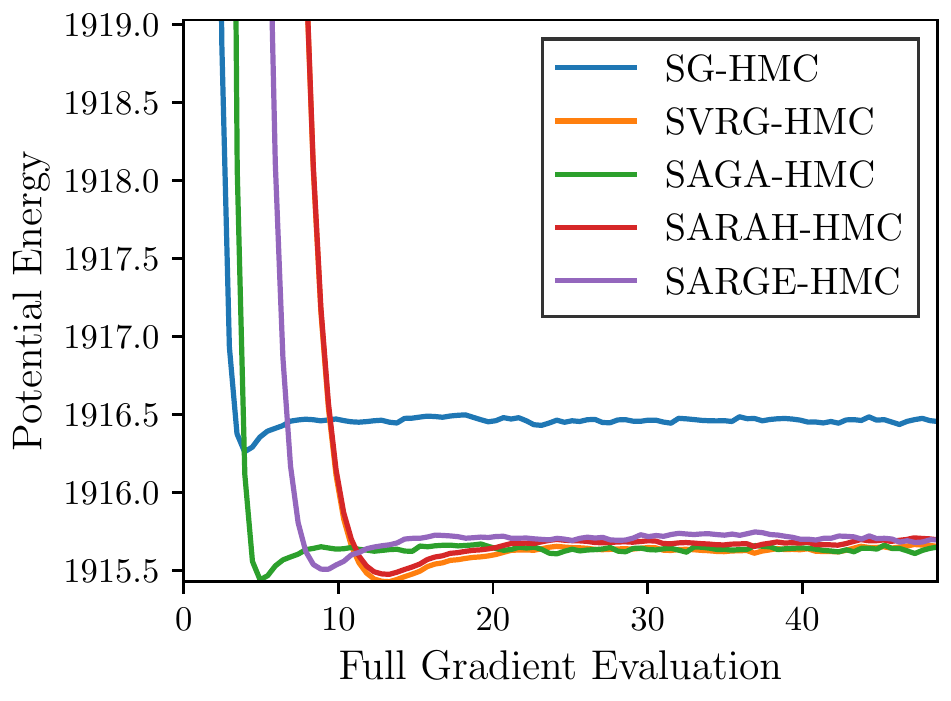}
		\caption{mushrooms}
	\end{subfigure}
	\caption{\small Mean potential energy of different algorithms on training datasets for logistic regression task.} 
	\label{fig:potential}
\end{figure*}
\begin{figure*}[t]
	\centering
	\begin{subfigure}[b]{0.24\textwidth}
		\centering
		\includegraphics[width=\textwidth]{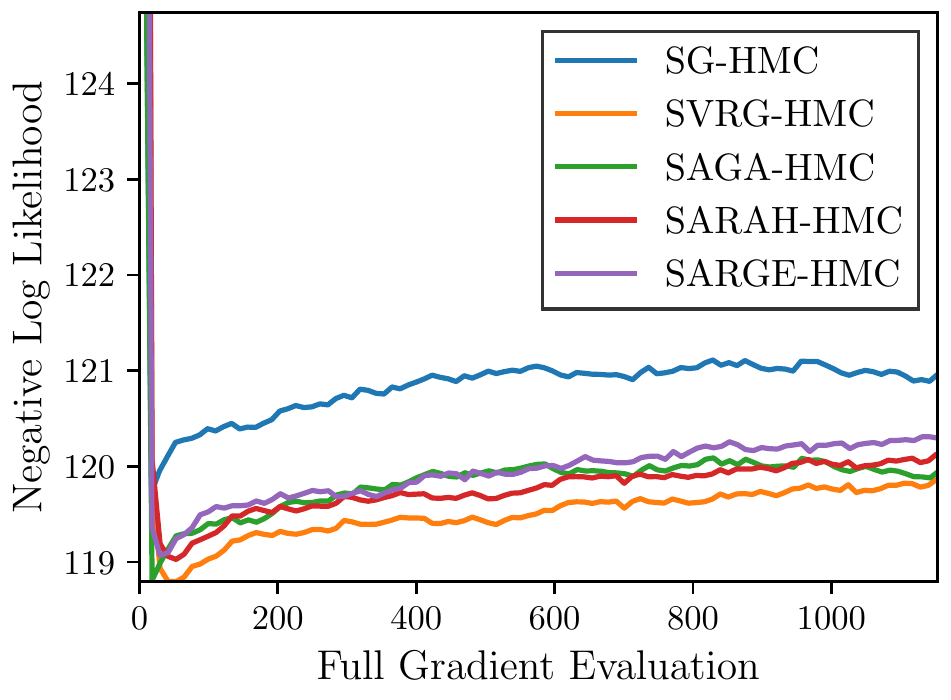}
		\caption{australian}
	\end{subfigure}
	\begin{subfigure}[b]{0.24\textwidth}  
		\centering
		\includegraphics[width=\textwidth]{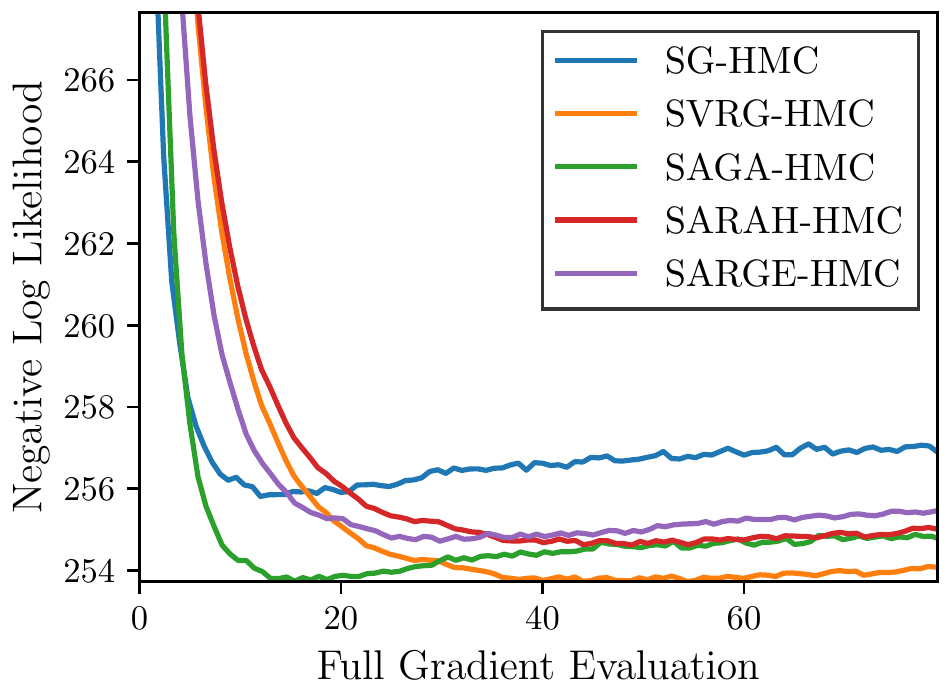}
		\caption{german}
	\end{subfigure}
	\begin{subfigure}[b]{0.24\textwidth}
		\centering
		\includegraphics[width=\textwidth]{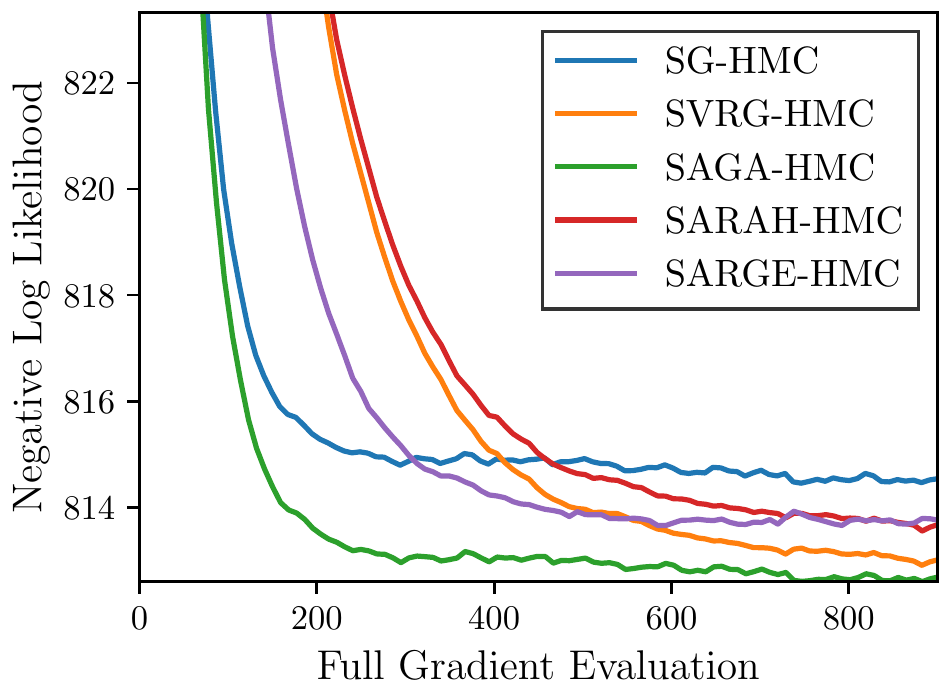}
		\caption{phishing}
	\end{subfigure}
	\begin{subfigure}[b]{0.24\textwidth}
		\centering
		\includegraphics[width=\textwidth]{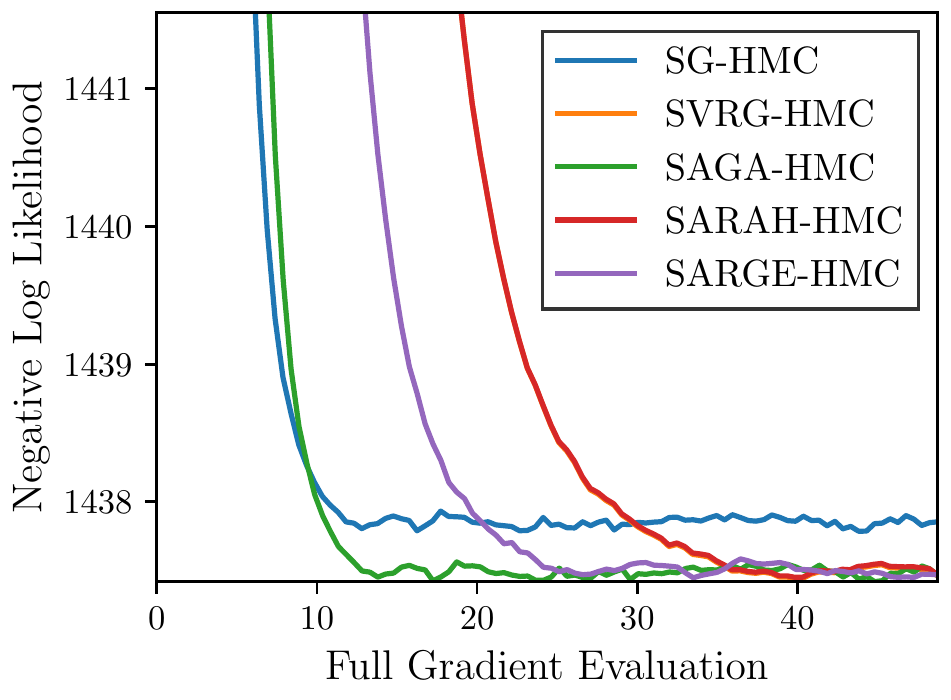}
		\caption{mushrooms}
	\end{subfigure}
	\caption{\small Negative log-likelihood energy of different algorithms on testing datasets for logistic regression task.} 
	\label{fig:likelihood}
\end{figure*}

\section{Experimental Results}

In this section, we will evaluate our algorithms on both synthetic data and real-world benchmark data.
During the following experiments, SVRG, SAGA, SARAH, and SARGE will be incorporated with our framework for evaluations. The corresponding algorithms are called as SVRG-HMC, SAGA-HMC, SARAH-HMC, and SARGE-HMC, respectively.
\subsection{Synthetic Data}
Following previous works \cite{Chen2014,Zou2018}, 
we use quadratic function as potential energy for our synthetic data.
The potential energy function can be decomposed into $N$ components $f_i(\bm{x}) = \frac{1}{N}(\bm{d}_i-\bm{x})^\top \Sigma^{-1} (\bm{d}_i-\bm{x})$, where $\bm{x} \in \mathbb{R}^d$ is the parameter to sample and $\bm{d}_i \in \mathbb{R}^d$ is the $i$-th data element generated from $\bm{d}_i \sim \mathcal{N}(\bm{2}, 2 \bm{I}_{d\times d})$.
$\Sigma^{-1}$ is a random positive-definite matrix whose maximum eigenvalue is $L$ and the minimum eigenvalue is $m$. 
Clearly, the invariant distribution is a Gaussian distribution with mean as average of $\bm{d}_i$ and covariance as $\Sigma$. During the experiment, we set $L= 10, d=5,N=1000$.

We set uniform step size for different algorithms and set batch size as $b=1$. We estimate the mean potential energy by accumulating for ten million iterations after burn-in of ten thousand iterations. 
We report the MSE of mean potential energy and gradient MSE of different algorithms in \Cref{tab:syn_mse}.

Firstly, all variance reduction methods based HMC enjoy more accurate gradient estimation and have smaller sampling error than SG-HMC.
Due to the simpleness of the quadratic potential function, SVRG and SARGE can eliminate the gradient error, thus the sampling error of SVRG-HMC and SARGE-HMC is exactly the same as full gradient HMC.
We also notice that SARGE-HMC achieves smaller gradient error than SAGA-HMC, but has larger MSE on potential energy.
This supports our theoretical analysis: the biased gradient estimator based HMC could be worse than the unbiased one even if with smaller gradient MSE.
\begin{table}
	\caption{The summary of different datasets used in our experiments.}
	\label{tab:datasets_nd}
	\begin{center}
		\begin{small}
			\begin{tabular}{lcccc}
				\toprule
				Dataset &  australian & german & phishing  &  mushrooms \\
				\midrule
				$N$& 690 & 1000 & 11055 & 8124\\
				$d$& 14 & 24 & 68 & 112
				\\
				\bottomrule
			\end{tabular}
		\end{small}
	\end{center}
\end{table}
\subsection{Bayesian Logistic Regression}
We further conduct experiments in Bayesian Logistic Regression on multiple real-world benchmark datasets. 

Typically in logistic regression, we are given a group of pairs $\{\bm{a}_i, y_i\}$, where  $\bm{a}_i$ is the feature vector and $y_i$ is binary label for each sample. We assume the likelihood function has the form $p(y_i | \bm{a}_i, \bm{x}) = \frac{1}{1+\exp(-y_i \bm{a}_i^\top \bm{x} )}$, then we have the posterior of parameter $\bm{x}$ as:
$
p^*(\bm{x}) = p_{prior}(\bm{x}) \prod\limits_{i =1}^{N} p(y_i | \bm{a}_i, \bm{x}).
$

Here we use the Gaussian distribution $\mathcal{N}(\bm{0}, m^{-1} \bm{I}_{d\times d})$ as prior. 
The corresponding potential energy function $f(\bm{x})$ can be written as:
$$
f(\bm{x}) = \frac{m}{2} \norm{\bm{x}}^2 + \sum\limits_{i=1}^{N} \log(1+\exp(-y_i \bm{a}_i^\top \bm{x} ))\,.
$$

We choose four benchmark datasets from LIBSVM \cite{Chang2011}. 
Their dimensionality and sample size are summarized in \Cref{tab:datasets_nd}. 
We divide the data into training set and testing set evenly.
The batch size is set to 1 for all algorithms.
Since it is computationally intractable to calculate the 2-Wasserstein distance in high dimensional space, we choose to record the average potential energy for training dataset and negative log-likelihood for testing dataset along the sample path to reflect the convergence and sampling error.
In order to control the influence of step size on the sampling error, we choose a uniform step size for all algorithms.
We also set small batch size $b=1$ for all algorithms.
We run each algorithm several thousand times and report the average result to reduce the noise.
The full gradient method is not examined due to slow convergence.
The potential energy for training dataset is shown in \Cref{fig:potential} and the negative log-likelihood for testing dataset is shown in \Cref{fig:likelihood}.

Obviously all variance reduced methods based HMC achieve lower mean potential energy compared to the SG-HMC, which indicates that our HMC framework can approximate the posterior much better than SG-HMC. 
We also notice that all algorithms take similar number of iterations to reach equilibrium. 
However, SVRG-HMC and SARAH-HMC take three gradient queries for each iteration on average and SARGE-HMC takes two gradient queries for each iteration. 
Therefore, these methods need more gradient evaluation for burn-in than SAGA-HMC and SG-HMC.

\begin{figure}[ht]
	\vspace*{-7pt}
	\centering
	\begin{subfigure}[b]{0.23\textwidth}
		\centering
		\includegraphics[width=\textwidth]{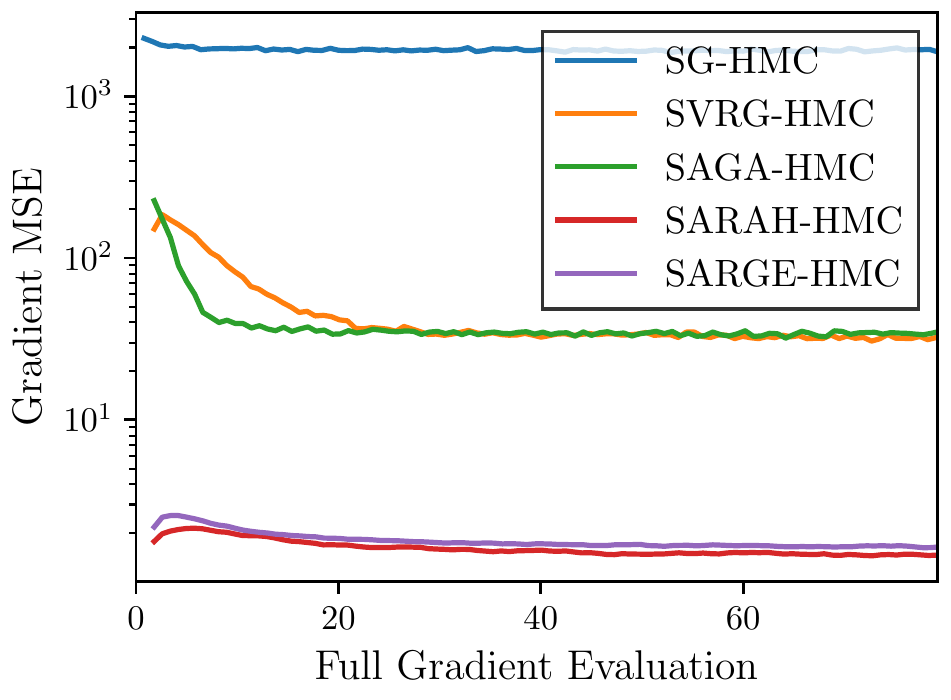}
		\caption{german}
	\end{subfigure}
	\begin{subfigure}[b]{0.23\textwidth}
		\centering
		\includegraphics[width=\textwidth]{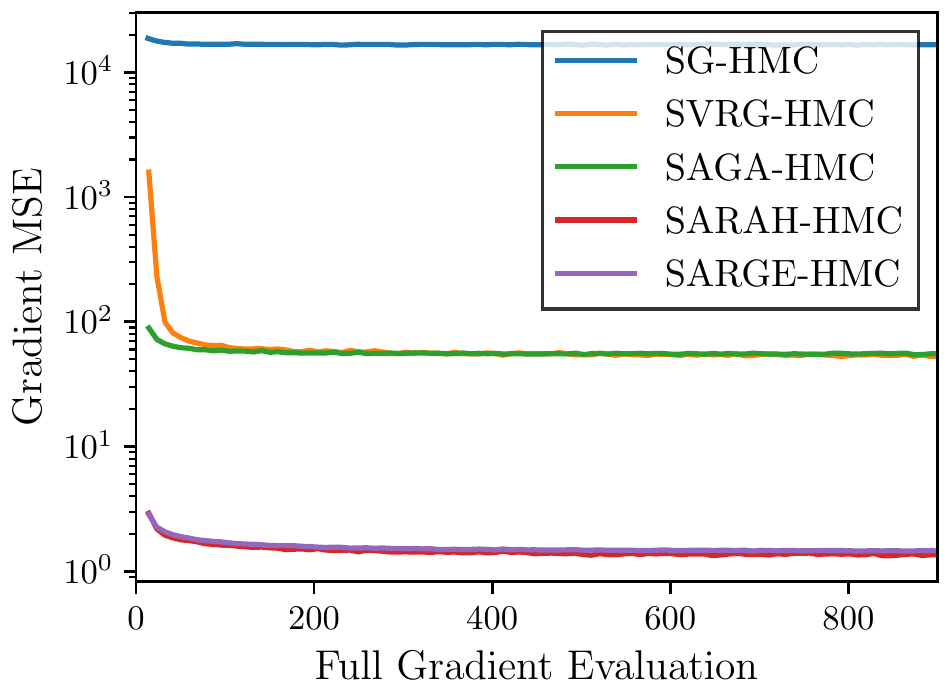}
		\caption{phishing}
	\end{subfigure}
	\caption{Gradient MSE for different algorithms.} 
	\label{fig:gradienterror}
	\vspace*{-14pt}
\end{figure}
We also report the gradient MSE of different algorithms on \textit{german} and \textit{phishing} datasets in \Cref{fig:gradienterror}. The gradient MSE plots for the other two datasets are similar.
Clearly, the biased gradient estimator (SARAH and SARGE) based methods achieve best gradient estimation. However, according to the mean potential energy and the negative log-likelihood, SARAH-HMC and SARGE-HMC are slightly worse than SVRG-HMC and SAGA-HMC. 
This phenomenon is once again consistent with our theoretical analysis.

\section{Conclusion}
We proposed a new framework of variance-reduced Hamiltonian Monte Carlo (HMC) method for sampling from an $L$-smooth and $m$-strongly log-concave distribution. 
The popular variance-reduction techniques, such as SAGA, SVRG, SARAH, and SARGE, can be combined with our framework. 
We derived the theoretical guarantee for the convergence of our framework based on the MSEB property, and we showed that all variance reduction methods considered in this paper improve the gradient complexity compared to the full gradient and stochastic gradient HMC approaches.

\bibliographystyle{unsrtnat}
\bibliography{hmc}

\begin{appendices}
\input{appendix}
\end{appendices}

\end{document}

%% file: appendix.tex
\section{Proof of Main Theory}
Let $\Phi^t$ be the evolution operator of distribution regarding to the original Hamilton dynamics \cref{eq:sde_hd}. \\
Let $\Phi^t_{\nabla}$ be the evolution operator regarding to the Hamilton dynamics conditioned on full gradient \cref{eq:dis_sde_hd}. \\
Let $\Phi^t_{\tilde{\nabla}}$ be the evolution operator regarding to the Hamilton dynamics conditioned on MSEB gradient estimator \cref{eq:dis_sde_vrhd}. \\
\begin{equation}
\label{eq:dis_sde_vrhd}
\small
d\tilde{\bm{V}}_t' = - \tilde{\nabla}_k dt -\gamma \xi \tilde{\bm{V}}_t' dt  + \sqrt{2\gamma} d\bm{B}_t, \;
d\tilde{\bm{X}}_t' = \xi \tilde{\bm{V}}_t' dt.
\end{equation}
Let $\Phi^t_{\mathbb{E}\tilde{\nabla}}$ be the evolution operator regarding to the Hamilton dynamics conditioned on conditional expectation of MSEB gradient estimator \cref{eq:dis_sde_mean_vrhd}. \\
\begin{equation}
\label{eq:dis_sde_mean_vrhd}
\small
d\tilde{\bm{V}}_t'' = - \mathbb{E}_{k-1}\tilde{\nabla}_k dt -\gamma \xi \tilde{\bm{V}}_t'' dt  + \sqrt{2\gamma} d\bm{B}_t, \;
d\tilde{\bm{X}}_t'' = \xi \tilde{\bm{V}}_t'' dt.
\end{equation}

If the initial condition $(\bm{x}_k,\bm{v}_k)$ has the distribution $p_k$, then 
the distribution of 
$({\bm{X}}_t, {\bm{V}}_t)$ is 
$\Phi^t p_k$ and the distributions of 
$(\tilde{\bm{X}}_t, \tilde{\bm{V}}_t)$, 
$(\tilde{\bm{X}}_t', \tilde{\bm{V}}_t')$ and 
$(\tilde{\bm{X}}_t'', \tilde{\bm{V}}_t'')$ are 
$\Phi^t_{\nabla} p_k$, 
$\Phi^t_{\tilde{\nabla}} p_k$ and 
$\Phi^t_{\mathbb{E}\tilde{\nabla}} p_k$ respectively.
We also denote $\Phi^t \bm{x}_k$ and $\Phi^t \bm{v}_k$ as the stochastic variable $\bm{X}_t$ and $\bm{V}_t$ in \cref{eq:sde_hd} with initial value $\bm{x}_k$ $\bm{v}_k$. Similarly, $\Phi^t_{\nabla} \bm{x}_k$ and $\Phi^t_{\nabla} \bm{v}_k$ represent $\tilde{\bm{X}}_t$ and $\tilde{\bm{V}}_t$ in \cref{eq:dis_sde_hd} with initial value $\bm{x}_k$ $\bm{v}_k$.

\begin{lemma}\label{le:one_step_w2}
	Under same conditions of \cref{th:main}, we have
	
	\begin{equation}
	W_2^2(\Phi^h_{\tilde{\nabla}} q_k,\Phi^{h (k+1)} q^\ast) 
	\leq
	A + ( e^{-\frac{\delta}{4 \kappa}} W_2 (q_k, \Phi^{h k} q^\ast) + B )^2
	\end{equation}
	
	\begin{equation}
	A \leq \Theta \delta^4 (4 \norm{\bm{x}_0}^2 + \frac{6 \kappa d}{L}) = F_1 \Theta \frac{\delta^4}{4 L}
	\end{equation}
	
	\begin{equation}
	B \leq 
	(1-\rho_B) \sqrt{A} 
	+ \frac{\delta^2 ( \sqrt{15 \delta} + 5 \sqrt{3} ) \sqrt{F_2}}{60 \sqrt{L}}
	\end{equation}
	where $\delta = \gamma \xi h$.
\end{lemma}

\begin{proof}[Proof of \cref{le:one_step_w2}]
	\begin{equation}
	\begin{aligned}
	W_2^2(\Phi^h_{\tilde{\nabla}} q_k,\Phi^{h (k+1)} q^\ast) 
	= & \mathbb{E} \norm{
		\Phi^h_{\tilde{\nabla}} q_k - \Phi^h_{\mathbb{E}\tilde{\nabla}} q_k + \Phi^h_{\mathbb{E}\tilde{\nabla}} q_k - \Phi^{h (k+1)} q^\ast
	}^2 \\
	= & \mathbb{E} \norm{\Phi^h_{\tilde{\nabla}} q_k - \Phi^h_{\mathbb{E}\tilde{\nabla}} q_k}^2
	+ \mathbb{E} \norm{\Phi^h_{\mathbb{E}\tilde{\nabla}} q_k - \Phi^{h (k+1)} q^\ast}^2 \\
	& + 2 \mathbb{E} \inner{\Phi^h_{\tilde{\nabla}} q_k - \Phi^h_{\mathbb{E}\tilde{\nabla}} q_k}{\Phi^h_{\mathbb{E}\tilde{\nabla}} q_k - \Phi^{h (k+1)} q^\ast} \\
	= & \mathbb{E} \norm{\Phi^h_{\tilde{\nabla}} q_k - \Phi^h_{\mathbb{E}\tilde{\nabla}} q_k}^2
	+ \mathbb{E} \norm{\Phi^h_{\mathbb{E}\tilde{\nabla}} q_k - \Phi^{h (k+1)} q^\ast}^2 \\
	& + 2 \mathbb{E} \mathbb{E}_{k-1} \inner{\Phi^h_{\tilde{\nabla}} q_k - \Phi^h_{\mathbb{E}\tilde{\nabla}} q_k}{\Phi^h_{\mathbb{E}\tilde{\nabla}} q_k - \Phi^{h (k+1)} q^\ast} \\
	= & \mathbb{E} \norm{\Phi^h_{\tilde{\nabla}} q_k - \Phi^h_{\mathbb{E}\tilde{\nabla}} q_k}^2
	+ \mathbb{E} \norm{\Phi^h_{\mathbb{E}\tilde{\nabla}} q_k - \Phi^{h (k+1)} q^\ast}^2
	\end{aligned}
	\end{equation}
	
	According to \cref{le:one_step_ge_mean_ge}, we can bound the first term as follows.
	\begin{equation}
		\mathbb{E} \norm{\Phi^h_{\tilde{\nabla}} q_k - \Phi^h_{\mathbb{E}\tilde{\nabla}} q_k}^2
		\leq \frac{\delta^2}{4 L^2} \mathbb{E} \norm{\tilde{\nabla}_k-\mathbb{E}\tilde{\nabla}_k}^2
	\end{equation}
	
	We further relax them term $ \mathbb{E} \norm{\tilde{\nabla}_k-\mathbb{E}\tilde{\nabla}_k}^2$ into $ \mathbb{E} \norm{\tilde{\nabla}_k-\nabla f(\bm{x}_k)}^2$ whose upper bound can be found at \cref{le:discrete_final}.
	
	We split the second term further.
	
	\begin{equation}\label{eq:one_step_mean_eg_c}
	\begin{aligned}
	\mathbb{E} \norm{\Phi^h_{\mathbb{E}\tilde{\nabla}} q_k - \Phi^{h (k+1)} q^\ast}^2
	= & \mathbb{E} \lVert 
	\Phi^h_{\mathbb{E}\tilde{\nabla}} q_k -\Phi^h_{\nabla} q_k \\
	& +\Phi^h_{\nabla} q_k- \Phi^h q_k \\
	& +\Phi^h q_k- \Phi^{h (k+1)} q^\ast 
	\rVert_2^2 \\
	\leq & ( \sqrt{\mathbb{E} \norm{\Phi^h_{\mathbb{E}\tilde{\nabla}} q_k -\Phi^h_{\nabla} q_k}^2} \\
	& + \sqrt{\mathbb{E} \norm{\Phi^h_{\nabla} q_k- \Phi^h q_k}^2} \\
	& + \sqrt{\mathbb{E} \norm{\Phi^h q_k- \Phi^{h (k+1)} q^\ast}^2}	
	)^2
	\end{aligned}
	\end{equation}
	
	The first term in the last line of \cref{eq:one_step_mean_eg_c} is controlled in \cref{le:one_step_mean_eg_g}.
	
	We split the second term in the last line of \cref{eq:one_step_mean_eg_c} as follows.
	
	\begin{equation}
	\sqrt{\mathbb{E} \norm{\Phi^h_{\nabla} q_k- \Phi^h q_k}^2}
	\leq 2 \sqrt{\mathbb{E} \norm{\Phi^h_{\nabla} x_k- \Phi^h x_k}^2}
	+ \sqrt{\mathbb{E} \norm{\Phi^h_{\nabla} v_k- \Phi^h v_k}^2}
	\end{equation}
	
	In \cref{le:one_step_g_c_on_momentum} , we show that both these two terms can be controlled by momentum $\max_{r<h} \mathbb{E} \norm{ \bm{V}_r}^2 = \max_{r<h} \mathbb{E} \norm{\Phi^h \bm{v}_k}^2$ as follows.

	\begin{equation}
	\begin{aligned}
	\sqrt{\mathbb{E} \norm{\Phi^h_{\nabla} q_k- \Phi^h q_k}^2}
	\leq & 
	2 \sqrt{\mathbb{E} \norm{\Phi^h_{\nabla} x_k- \Phi^h x_k}^2}
	+ \sqrt{\mathbb{E} \norm{\Phi^h_{\nabla} v_k- \Phi^h v_k}^2} \\
	\leq & 
	\frac{2}{\sqrt{15}} h^3 L^3 \sqrt{\mathbb{E}\norm{\Phi^h_{\nabla} v_k}^2}
	+ \frac{1}{\sqrt{3}} h^2 L^2 \sqrt{\mathbb{E}\norm{\Phi^h_{\nabla} v_k}^2}
	\end{aligned}
	\end{equation}
	
	By assuming small step size, we can also derive an upper bound for the momentum in \cref{le:bounded_momentum}.
	
	The third term in the last line of \cref{eq:one_step_mean_eg_c} decreases due to the contraction property of HMC on a strongly log-concave distribution. According to \citep[Theorem 5]{Cheng2018a}, the following inequality holds.
	
	\begin{equation}
	\begin{aligned}
	\mathbb{E} \norm{\Phi^h q_k- \Phi^{h (k+1)} q^\ast}^2
	\leq & W_2^2 (\Phi^h q_k, \Phi^{h (k+1)} q^\ast) \\
	\leq & e^{-\frac{\delta}{2 \kappa}} W_2^2 (q_k, \Phi^{h k} q^\ast)
	\end{aligned}
	\end{equation}
	
	Combining all above upper bounds for each term give rise to the final upper bound.
\end{proof}

\begin{proof}[Proof of \cref{th:main}]
	By Lemma 7 of \cite{Dalalyan2019}, if $x_{k+1}^2\leq ((1-\alpha) x_k + B)^2 + A$, then 
	\begin{equation}\label{eq:dala_ieq}
	x_k 
	\leq (1-\alpha)^k x_0+\frac{B}{\alpha} + \frac{A}{B+\sqrt{\alpha (2-\alpha) A}} 
	\leq (1-\alpha)^k x_0+\frac{B}{\alpha} + \frac{\sqrt{A}}{\sqrt{\alpha}} 
	\end{equation}
	
	Because our step size is small enough, we have
	$$e^{-\frac{\delta}{4 \kappa}} < 1- \frac{\delta}{8\kappa}$$
	
	We apply inequality \cref{eq:dala_ieq} into \cref{le:one_step_w2} to finish the proof.
	\begin{equation}
	\begin{aligned}
	W_2(\Phi^h_{\tilde{\nabla}} q_k,\Phi^{h (k+1)} q^\ast) 
	\leq &
	e^{-\frac{k \delta}{4 \kappa}} W_2(q_0, q^\ast) 
	+ \frac{8 \kappa}{\delta} B
	+ \frac{\sqrt{8 \kappa}}{\sqrt{\delta}} \sqrt{A} \\
	\leq & 
	e^{-\frac{k h m}{2}} W_2(q_0, q^*) + 8 \sqrt{L F_2} \kappa h  \\
	& + 4 \sqrt{\Theta F_1} \left(2(1-\rho_B)\sqrt{L} \kappa h + L \sqrt{\kappa} h^\frac{3}{2}\right)
	\end{aligned}
	\end{equation}
\end{proof}

\section{Technical Lemmas}

\begin{lemma}\label{le:one_step_any}
	In \cref{eq:dis_sde_hd}, if we choose two different gradient $\tilde{\nabla}_k$ and $\nabla_k$ to generate two different SDE with same initial distribution $q_k$, the Wasserstein distance of distribution of $\Phi^h_{\tilde{\nabla}} q_k$ and $\Phi^h_{\nabla} q_k$ can be upper bounded by the gradient difference in the following way. 
	
	\begin{equation}
	\mathbb{E} \norm{\Phi^h_{\tilde{\nabla}} q_k - \Phi^h_{\nabla} q_k}^2
	\leq \frac{\delta^2}{4 L^2} \mathbb{E} \norm{\tilde{\nabla}_k-\nabla_k}^2
	\end{equation}
	
	The above inequality holds true for all positive step size.
\end{lemma}

\begin{proof}[Proof of \cref{le:one_step_any}]
	\begin{equation}
	\begin{aligned}
	\mathbb{E} \norm{\Phi^h_{\tilde{\nabla}} q_k - \Phi^h_{\nabla} q_k}^2
	= & \mathbb{E} \norm{\Phi^h_{\tilde{\nabla}} \bm{x}_k - \Phi^h_{\nabla} \bm{x}_k}^2 \\
	& + \mathbb{E} \norm{\Phi^h_{\tilde{\nabla}} \bm{x}_k - \Phi^h_{\nabla} \bm{x}_k + \Phi^h_{\tilde{\nabla}} \bm{v}_k - \Phi^h_{\nabla} \bm{v}_k}^2 \\
	= & \left(\frac{\nabla_k \left(h - \frac{1 - e^{- \gamma h \xi}}{\gamma \xi}\right)}{\gamma} - \frac{\tilde{\nabla}_k \left(h - \frac{1 - e^{- \gamma h \xi}}{\gamma \xi}\right)}{\gamma}\right)^{2} \\
	& + \left(\frac{\nabla_k \left(h - \frac{1 - e^{- \gamma h \xi}}{\gamma \xi}\right)}{\gamma} + \frac{\nabla_k \left(1 - e^{- \gamma h \xi}\right)}{\gamma \xi} \right. \\ 
	& \left.- \frac{\tilde{\nabla}_k \left(h - \frac{1 - e^{- \gamma h \xi}}{\gamma \xi}\right)}{\gamma} - \frac{\tilde{\nabla}_k \left(1 - e^{- \gamma h \xi}\right)}{\gamma \xi}\right)^{2} \\
	= & \frac{(\tilde{\nabla}_k-\nabla_k)^{2} e^{- 2 \gamma h \xi}}{\gamma^{4} \xi^{2}} \times \left(\left(\gamma h \xi e^{\gamma h \xi} - e^{\gamma h \xi} + 1\right)^{2} \right.\\
	& \left. + \left(- \gamma h \xi e^{\gamma h \xi} + \gamma \left(1 - e^{\gamma h \xi}\right) + e^{\gamma h \xi} - 1\right)^{2}\right) \\
	= & \frac{(\tilde{\nabla}_k-\nabla_k)^{2} \left(\left(\delta^{2} + 1\right) e^{2 \delta} - 2 e^{\delta} + 1\right) e^{- 2 \delta}}{8 L^{2}} \\
	\leq & \frac{(\tilde{\nabla}_k-\nabla_k)^{2} \delta^{2}}{4 L^{2}}
	\end{aligned}
	\end{equation}
	The last inequality doesn't depend on any assumption of small step size.
\end{proof}

\begin{lemma}\label{le:one_step_ge_mean_ge}
	\begin{equation}
	\begin{aligned}
	\mathbb{E} \norm{\Phi^h_{\tilde{\nabla}} q_k - \Phi^h_{\mathbb{E}\tilde{\nabla}} q_k}^2
	= & \mathbb{E} \norm{(d\tilde{\bm{X}}_h'-d\tilde{\bm{X}}_h'',d\tilde{\bm{X}}_h'-d\tilde{\bm{X}}_h'' + d\tilde{\bm{V}}_h'-d\tilde{\bm{V}}_h'')}^2 \\
	\leq & \frac{\delta^2}{4 L^2} \mathbb{E} \norm{\tilde{\nabla}_k-\mathbb{E}\tilde{\nabla}_k}^2
	\end{aligned}
	\end{equation}
\end{lemma}

\begin{proof}[Proof of \cref{le:one_step_ge_mean_ge}]
	This is just a special case of \cref{le:one_step_any}.
\end{proof}

\begin{lemma}\label{le:one_step_mean_eg_g}
	\begin{align}
	\mathbb{E} \norm{\Phi^h_{\mathbb{E}\tilde{\nabla}} q_k -\Phi^h_{\nabla} q_k}^2
	\leq & \frac{\delta^2}{4 L^2} \mathbb{E}  \norm{\mathbb{E}_{k-1}\tilde{\nabla}_k - \nabla f(\bm{x}_k)}^2 \\
	\leq & \frac{\delta^2}{4 L^2} (1-\rho_B)^2 \mathbb{E}  \norm{\tilde{\nabla}_{k-1} - \nabla f(\bm{x}_{k-1})}^2
	\end{align}
\end{lemma}
\begin{proof}[Proof of \cref{le:one_step_mean_eg_g}]
	The first inequality comes from \cref{le:one_step_any}, and the second inequality comes from MSEB property.
\end{proof}

\begin{lemma}\label{le:one_step_g_c_on_momentum}
	\begin{equation}
	\mathbb{E} \norm{\Phi^h_{\nabla} \bm{v}_k- \Phi^h \bm{v}_k}^2
	\leq \frac{1}{3} h^4 L^4 \max_{r<h} \mathbb{E} \norm{ \bm{V}_r}^2
	\end{equation}
	\begin{equation}
	\mathbb{E} \norm{\Phi^h_{\nabla} \bm{x}_k- \Phi^h \bm{x}_k}^2
	\leq  \frac{1}{15} h^6 L^6 \max_{r<h} \mathbb{E} \norm{ \bm{V}_r}^2
	\end{equation}
\end{lemma}
\begin{proof}[Proof of \cref{le:one_step_g_c_on_momentum}]
	\begin{equation}
	\begin{aligned}
	\mathbb{E} \norm{\Phi^h_{\nabla} \bm{v}_k- \Phi^h \bm{v}_k}^2
	\leq & \mathbb{E} \norm{
		\int_{0}^{h} e^{-\gamma \xi (h-s)} (\nabla f(\bm{X}_s) - \nabla f(\bm{x}_k)) ds
	}^2 \\
	\leq & h \int_{0}^{h} \mathbb{E} \norm{e^{-\gamma \xi (h-s)} (\nabla f(\bm{X}_s) - \nabla f(\bm{x}_k))}^2 ds \\
	\leq & h L^2 \int_{0}^{h} \mathbb{E} \norm{\bm{X}_s -\bm{x}_k }^2 ds \\
	\leq & h L^2 \int_{0}^{h} \mathbb{E} \norm{ \int_{0}^{s} \xi \bm{V}_r dr }^2 ds \\
	\leq & h L^2 \xi^2 \int_{0}^{h} s \int_{0}^{s} \mathbb{E} \norm{ \bm{V}_r}^2 dr ds \\
	\leq & \frac{1}{3} h^4 L^4 \max_{r<h} \mathbb{E} \norm{ \bm{V}_r}^2
	\end{aligned}
	\end{equation}
	
	\begin{equation}
	\begin{aligned}
	\mathbb{E} \norm{\Phi^h_{\nabla} \bm{x}_k- \Phi^h \bm{x}_k}^2
	= & \mathbb{E} \norm{\int_{0}^{h} \xi (\Phi^s_{\nabla} \bm{v}_k- \Phi^s \bm{v}_k) ds}^2 \\
	\leq & h \xi^2 \int_{0}^{h} \mathbb{E} \norm{\Phi^s_{\nabla} \bm{v}_k- \Phi^s \bm{v}_k}^2 ds \\
	\leq & \frac{1}{15} h^6 L^6 \max_{r<h} \mathbb{E} \norm{ \bm{V}_r}^2
	\end{aligned}
	\end{equation}
\end{proof}
\begin{lemma}\label{le:bounded_momentum}
With small step size assumption, we have the momentum bounded as follows.
\begin{equation}
\mathbb{E} \norm{\Phi^h \bm{v}_k}^2 \leq
97 \norm{\bm{x}_0}^2 + \frac{181 \kappa d}{L}
\end{equation}
\end{lemma}
\begin{proof}[Proof of \cref{le:bounded_momentum}]
We control the momentum in a recursive way.

First we show that $\mathbb{E} \norm{\bm{V}_h}^2$ and $\mathbb{E} \norm{\bm{X}_h}^2$ can be controlled by step change $\mathbb{E} \norm{\Phi^h \bm{v}_k - \bm{v}_k}^2$ and $\mathbb{E} \norm{\Phi^h \bm{x}_k - \bm{x}_k}^2$, and then we show that the reverse is also true.

\begin{equation}
\begin{aligned}
\mathbb{E} \norm{\bm{V}_h}^2
= & \mathbb{E} \norm{\Phi^h \bm{v}_k}^2 \\
\leq & 2 \mathbb{E} \norm{\Phi^h \bm{v}_k - \bm{v}_k}^2 + 2 \mathbb{E} \norm{\bm{v}_k}^2
\end{aligned}
\end{equation}

\begin{equation}
\begin{aligned}
\mathbb{E} \norm{\bm{X}_h}^2
= & \mathbb{E} \norm{\Phi^h \bm{x}_k}^2 \\
\leq & 2 \mathbb{E} \norm{\Phi^h \bm{x}_k - \bm{x}_k}^2 + 2 \mathbb{E} \norm{\bm{x}_k}^2
\end{aligned}
\end{equation}

\begin{equation}
\begin{aligned}
\mathbb{E} \norm{\Phi^h \bm{v}_k - \bm{v}_k}^2
= & \mathbb{E} \norm{
	\int_{0}^{h} e^{-\gamma \xi (h-s)} \nabla f(\bm{X}_s) ds
}^2
+ 2 \gamma \mathbb{E} \norm{
	\int_{0}^{h} e^{-\gamma \xi (h-s)} d\bm{B}_s
}^2 \\
\leq & h \int_{0}^{h} \mathbb{E} \norm{\nabla f(\bm{X}_s)}^2 ds + \frac{1}{\xi} (1-e^{-\gamma \xi h}) \\
\leq & h L^2 \int_{0}^{h} \mathbb{E} \norm{\bm{X}_s}^2 ds + \frac{1}{\xi} (1-e^{-\gamma \xi h}) \\
\leq & h^2 L^2 \max_{r<h} \mathbb{E} \norm{\bm{X}_r}^2 ds + \gamma h
\end{aligned}
\end{equation}

\begin{equation}
\begin{aligned}
\mathbb{E} \norm{\Phi^h \bm{x}_k - \bm{x}_k}^2
= & \mathbb{E} \norm{ \int_{0}^{h} \xi \bm{V}_s ds }^2 \\
\leq & h \xi^2 \int_{0}^{h} \mathbb{E} \norm{\bm{V}_s}^2 ds \\
\leq & h^2 L^2 \max_{r<h} \mathbb{E} \norm{\bm{V}_r}^2 ds
\end{aligned}
\end{equation}

Combine the above four equation, we can see that

\begin{equation}
\begin{aligned}
\mathbb{E} \norm{\Phi^h \bm{v}_k}^2 \leq &
2 h^2 L^2 \max_{r<h} \mathbb{E} \norm{\Phi^r \bm{x}_k}^2 + 2 \gamma h + 2 \mathbb{E} \norm{\bm{v}_k}^2 \\
\mathbb{E} \norm{\Phi^h \bm{x}_k}^2 \leq &
2 h^2 L^2 \max_{r<h} \mathbb{E} \norm{\Phi^r \bm{v}_k}^2 + 2 \mathbb{E} \norm{\bm{x}_k}^2
\end{aligned}
\end{equation}

This further imply following inequality. 

\begin{equation}
\mathbb{E} \norm{\Phi^h \bm{v}_k}^2 \leq 
4 h^4 L^4 \max_{r<h} \mathbb{E} \norm{\Phi^r \bm{v}_k}^2
+ 4 h^2 L^2 \mathbb{E} \norm{\bm{x}_k}^2
+ 2 \gamma h 
+ 2 \mathbb{E} \norm{\bm{v}_k}^2
\end{equation}

We finish the proof by applying Gronwall's inequality and substitute $\mathbb{E} \norm{\bm{v}_k}^2$ and $\mathbb{E} \norm{\bm{x}_k}^2$ with their upper bound in \cref{le:discrete_final}.
\end{proof}

\begin{lemma}\label{le:discrete_final}
	With small enough step size $\delta$ satisfying 
	$\delta \leq \frac{1}{5 \kappa} \min(1,\frac{1}{\sqrt{\Theta}})$, the following inequalities holds.
	
	\begin{equation}
		\begin{aligned}
& 		\max_k \mathbb{E} \left( E(\bm{x}_k, \bm{v}_k) \right) \leq 24 \norm{\bm{x}_0}^2 + \frac{45 \kappa d}{L}
		\\
&		\max_k \mathbb{E} \norm{\bm{x}_k}^2 \leq 24 \norm{\bm{x}_0}^2 + \frac{45 \kappa d}{L}
		\\
&		\max_k \mathbb{E} \norm{\bm{v}_k}^2 \leq 48 \norm{\bm{x}_0}^2 + \frac{89 \kappa d}{L}
		\\
&		\max_k \mathbb{E} \norm{\nabla f(\bm{x}_k)}^2 \leq 24 L^{2} \norm{\bm{x}_0}^2 + 45 L \kappa d
		\\
&		\max_k \mathbb{E} \norm{\tilde{\nabla}_k-\nabla f(\bm{x}_k)}^2 \leq 13 L^{2} \Theta \delta^{2} \norm{\bm{x}_0}^2 + 24 L \Theta \delta^{2} \kappa d
		\\
&		\max_k Q_k \leq 13 L^{2} \delta^{2} \norm{\bm{x}_0}^2 + 24 L \delta^{2} \kappa d
		\end{aligned}
	\end{equation}
	where $E(\bm{x},\bm{v}) = \norm{\bm{x}}^2 + \norm{\bm{x} + \frac{2}{\gamma}\bm{v}}^2 + \frac{8}{\xi \gamma^2} (f(\bm{x})-f(\bm{x}^\ast))$ is the Lyapunov function.
\end{lemma}
\begin{proof}[Proof of \cref{le:discrete_final}]
	\cref{le:stage1_dE,le:stage1_momentum,le:stage1_theta,le:stage1_Q} show preliminary results of upper bounds.
	
	We further control coefficients in \cref{le:stage1_dE,le:stage1_Q}.
	If we have $\delta \leq \frac{17}{32}$, we can relax the coefficients of \cref{eq:stage1_dE,eq:stage1_Q} into 
	
	\begin{equation}\label{eq:stage3_dE}
	\begin{aligned}
		\max_k \mathbb{E} \left( E(\bm{x}_k, \bm{v}_k) \right) \leq & \frac{5 \delta \kappa \max_k \mathbb{E} \norm{\bm{x}_k}^2}{2} + \max_k \mathbb{E} \norm{\tilde{\nabla}_k-\nabla f(\bm{x}_k)}^2 u_{135} \\
		& + 6 \norm{\bm{x}_0}^2 + d u_{138}
	\end{aligned}
	\end{equation}
	\begin{equation}\label{eq:stage3_Q}
	\begin{aligned}
		\max_k Q_k \leq & \frac{L^{2} \delta^{3} \max_k \mathbb{E} \norm{\bm{x}_k}^2}{8} + \frac{L^{2} \delta^{2} \max_k \mathbb{E} \norm{\bm{v}_k}^2}{4} + \frac{L \delta^{3} d}{6} \\
		& + \frac{5 \delta^{3} \max_k \mathbb{E} \norm{\tilde{\nabla}_k-\nabla f(\bm{x}_k)}^2}{64}
	\end{aligned}
	\end{equation}
	
	Variables $u_i$ are used to simplify the formula. The definition of $u_i$ can be found at the end of this section.
	
	By applying \cref{eq:stage1_momentum_x,eq:stage1_momentum_v} into \cref{eq:stage3_dE}, we can show that
	\begin{equation}\label{eq:stage4_dE_pre}
	\begin{aligned}
		\max_k \mathbb{E} \left( E(\bm{x}_k, \bm{v}_k) \right) \leq \max_k \mathbb{E} \norm{\tilde{\nabla}_k-\nabla f(\bm{x}_k)}^2 u_{146} + 12 \norm{\bm{x}_0}^2 + d u_{147}
	\end{aligned}
	\end{equation}
	whenever $\frac{5 \delta \kappa}{2} \leq \frac{1}{2}$.
	
	By applying \cref{eq:stage4_dE_pre,eq:stage1_momentum_x,eq:stage1_momentum_v,eq:stage1_theta} into \cref{eq:stage1_theta}, we can show that 
	\begin{equation}\label{eq:stage4_theta}
	\begin{aligned}
	\max_k \mathbb{E} \norm{\tilde{\nabla}_k-\nabla f(\bm{x}_k)}^2 \leq \norm{\bm{x}_0}^2 u_{161} + d u_{160}
	\end{aligned}
	\end{equation}
	whenever $\frac{5 \Theta \delta^{3} \kappa^{2}}{4} + \frac{25 \Theta \delta^{3} \kappa}{16} + \frac{5 \Theta \delta^{3}}{64} + 5 \Theta \delta^{2} \kappa^{2} + \frac{25 \Theta \delta^{2} \kappa}{4} \leq \frac{1}{2}$.
	
	Applying \cref{eq:stage4_theta} back into \cref{eq:stage4_dE_pre} gives 
	\begin{equation}\label{eq:stage4_dE}
	\begin{aligned}
	\max_k \mathbb{E} \left( E(\bm{x}_k, \bm{v}_k) \right) \leq \norm{\bm{x}_0}^2 u_{159} + d u_{158}
	\end{aligned}
	\end{equation}
	
	We then apply \cref{eq:stage4_dE,eq:stage4_theta} into \cref{eq:stage1_momentum_x,eq:stage1_momentum_v,eq:stage1_momentum_f,eq:stage3_Q} and relax $\delta$ into lowest order and relax $\kappa$ into highest order to generate the final result.
\end{proof}

\begin{lemma}\label{le:stage1_dE}
	\begin{equation}\label{eq:stage1_dE}
	\begin{aligned}
	\max_k \mathbb{E} \left( E(\bm{x}_k, \bm{v}_k) \right) \leq & \max_k \mathbb{E} \norm{\bm{v}_k}^2 u_{101} + \max_k \mathbb{E} \norm{\bm{x}_k}^2 u_{102} \\
	& + \max_k \mathbb{E} \norm{\tilde{\nabla}_k-\nabla f(\bm{x}_k)}^2 u_{100} + 6 \norm{\bm{x}_0}^2 + d u_{104}
	\end{aligned}
	\end{equation}
	where expressions $u_i$ can be found at the end of this section.
\end{lemma}
\begin{proof}[Proof of \cref{le:stage1_dE}]
	\begin{equation}\label{eq:stage1_dE_proof}
	\begin{aligned}
	& \mathbb{E} E(\bm{x}_{k+1}, \bm{v}_{k+1}) - \mathbb{E} \left( E(\bm{x}_k, \bm{v}_k) \right) (1- \frac{\delta}{10 \kappa}) \\
	= & \mathbb{E} \inner{\bm{v}_k}{\bm{x}_k} u_{27} + 2 \mathbb{E} \inner{\bm{v}_{k+1}}{\bm{x}_{k+1}} + \mathbb{E} \norm{\bm{v}_k}^2 u_{18} + \mathbb{E} \norm{\bm{v}_{k+1}}^2 + \mathbb{E} \norm{\bm{x}_k}^2 u_{27} \\
	& + 2 \mathbb{E} \norm{\bm{x}_{k+1}}^2 + f(\bm{x}_k) u_{24} - \frac{\delta f(\bm{x}^\ast)}{5 L \kappa} + \frac{2 f(\bm{x}_{k+1})}{L} \\ 
	\leq & - \frac{\delta \mathbb{E} \inner{\bm{x}^\ast}{\bm{x}_k}}{5 \kappa} + \frac{\delta \norm{\bm{x}^\ast}^2}{10 \kappa} + \mathbb{E} \inner{\bm{v}_k}{\bm{x}_k} u_{27} + 2 \mathbb{E} \inner{\bm{v}_{k+1}}{\bm{x}_{k+1}} - 2 \mathbb{E} \inner{\bm{x}_{k+1}}{\bm{x}_k} \\
	& + \mathbb{E} \norm{\bm{v}_k}^2 u_{18} + \mathbb{E} \norm{\bm{v}_{k+1}}^2 + \mathbb{E} \norm{\bm{x}_k}^2 u_{32} + 3 \mathbb{E} \norm{\bm{x}_{k+1}}^2 - \frac{2 \mathbb{E} \inner{\nabla f(\bm{x}_k)}{\bm{x}_k}}{L} \\
	& + \frac{2 \mathbb{E} \inner{\nabla f(\bm{x}_k)}{\bm{x}_{k+1}}}{L} \\ 
	= & \mathbb{E} \inner{\bm{v}_k}{\bm{x}_k} u_{27} + 2 \mathbb{E} \inner{\bm{v}_{k+1}}{\bm{x}_{k+1}} - 2 \mathbb{E} \inner{\bm{x}_{k+1}}{\bm{x}_k} + \mathbb{E} \norm{\bm{v}_k}^2 u_{18} + \mathbb{E} \norm{\bm{v}_{k+1}}^2 \\
	& + \mathbb{E} \norm{\bm{x}_k}^2 u_{32} + 3 \mathbb{E} \norm{\bm{x}_{k+1}}^2 - \frac{2 \mathbb{E} \inner{\nabla f(\bm{x}_k)}{\bm{x}_k}}{L} + \frac{2 \mathbb{E} \inner{\nabla f(\bm{x}_k)}{\bm{x}_{k+1}}}{L} \\ 
	= & \frac{\delta \mathbb{E} \inner{\bm{v}_k}{\bm{x}_k}}{5 \kappa} + \frac{3 \delta \mathbb{E} \norm{\bm{x}_k}^2}{10 \kappa} + 2 \mathbb{E} \inner{\bm{e}^v_k}{\bm{e}^x_k} + \mathbb{E} \inner{\bm{e}^v_k}{\bm{v}_k} u_{62} + 2 \mathbb{E} \inner{\bm{e}^v_k}{\bm{x}_k} \\
	& + \mathbb{E} \inner{\bm{e}^x_k}{\bm{v}_k} u_{61} + 4 \mathbb{E} \inner{\bm{e}^x_k}{\bm{x}_k} + \mathbb{E} \inner{\nabla f(\bm{x}_k)}{\bm{v}_k} u_{63} + \mathbb{E} \inner{\tilde{\nabla}_k}{\bm{e}^v_k} u_{53} \\
	& + \mathbb{E} \inner{\tilde{\nabla}_k}{\bm{e}^x_k} u_{52} + \mathbb{E} \inner{\tilde{\nabla}_k}{\bm{v}_k} u_{60} + \mathbb{E} \inner{\tilde{\nabla}_k}{\nabla f(\bm{x}_k)} u_{56} + \mathbb{E} \norm{\bm{e}^v_k}^2 + 3 \mathbb{E} \norm{\bm{e}^x_k}^2 \\
	& + \mathbb{E} \norm{\bm{v}_k}^2 u_{64} + \mathbb{E} \norm{\tilde{\nabla}_k}^2 u_{47} - \frac{\delta \mathbb{E} \inner{\tilde{\nabla}_k}{\bm{x}_k}}{L} + \frac{2 \mathbb{E} \inner{\nabla f(\bm{x}_k)}{\bm{e}^x_k}}{L} \\ 
	= & \frac{\delta \mathbb{E} \inner{\bm{v}_k}{\bm{x}_k}}{5 \kappa} + \frac{3 \delta \mathbb{E} \norm{\bm{x}_k}^2}{10 \kappa} + \mathbb{E} \inner{\nabla f(\bm{x}_k)}{\bm{v}_k} u_{73} + \mathbb{E} \inner{\tilde{\nabla}_k-\nabla f(\bm{x}_k)}{\bm{v}_k} u_{60} \\
	& + \mathbb{E} \inner{\tilde{\nabla}_k-\nabla f(\bm{x}_k)}{\nabla f(\bm{x}_k)} u_{70} + \mathbb{E} \norm{\bm{v}_k}^2 u_{64} + \mathbb{E} \norm{\nabla f(\bm{x}_k)}^2 u_{72} \\
	& + \mathbb{E} \norm{\tilde{\nabla}_k-\nabla f(\bm{x}_k)}^2 u_{47} + d u_{68} - \frac{\delta \mathbb{E} \inner{\nabla f(\bm{x}_k)}{\bm{x}_k}}{L} - \frac{\delta \mathbb{E} \inner{\tilde{\nabla}_k-\nabla f(\bm{x}_k)}{\bm{x}_k}}{L} \\ 
	\leq & \frac{\delta \mathbb{E} \inner{\bm{v}_k}{\bm{x}_k}}{5 \kappa} - \frac{7 \delta \mathbb{E} \norm{\bm{x}_k}^2}{10 \kappa} + \mathbb{E} \inner{\nabla f(\bm{x}_k)}{\bm{v}_k} u_{73} + \mathbb{E} \inner{\tilde{\nabla}_k-\nabla f(\bm{x}_k)}{\bm{v}_k} u_{60} \\
	& + \mathbb{E} \inner{\tilde{\nabla}_k-\nabla f(\bm{x}_k)}{\nabla f(\bm{x}_k)} u_{70} + \mathbb{E} \norm{\bm{v}_k}^2 u_{64} + \mathbb{E} \norm{\nabla f(\bm{x}_k)}^2 u_{72} \\
	& + \mathbb{E} \norm{\tilde{\nabla}_k-\nabla f(\bm{x}_k)}^2 u_{47} + d u_{68} - \frac{\delta \mathbb{E} \inner{\tilde{\nabla}_k-\nabla f(\bm{x}_k)}{\bm{x}_k}}{L} \\ 
	\leq & \frac{\delta \mathbb{E} \inner{\bm{v}_k}{\bm{x}_k}}{5 \kappa} - \frac{7 \delta \mathbb{E} \norm{\bm{x}_k}^2}{10 \kappa} + \mathbb{E} \inner{\tilde{\nabla}_k-\nabla f(\bm{x}_k)}{\bm{v}_k} u_{60} + \mathbb{E} \norm{\bm{v}_k}^2 u_{87} \\
	& + \mathbb{E} \norm{\nabla f(\bm{x}_k)}^2 u_{86} + \mathbb{E} \norm{\tilde{\nabla}_k-\nabla f(\bm{x}_k)}^2 u_{80} + d u_{68} - \frac{\delta \mathbb{E} \inner{\tilde{\nabla}_k-\nabla f(\bm{x}_k)}{\bm{x}_k}}{L} \\ 
	\leq & \frac{\delta \mathbb{E} \inner{\bm{v}_k}{\bm{x}_k}}{5 \kappa} + \mathbb{E} \inner{\tilde{\nabla}_k-\nabla f(\bm{x}_k)}{\bm{v}_k} u_{60} + \mathbb{E} \norm{\bm{v}_k}^2 u_{87} + \mathbb{E} \norm{\bm{x}_k}^2 u_{89} \\
	& + \mathbb{E} \norm{\tilde{\nabla}_k-\nabla f(\bm{x}_k)}^2 u_{80} + d u_{68} - \frac{\delta \mathbb{E} \inner{\tilde{\nabla}_k-\nabla f(\bm{x}_k)}{\bm{x}_k}}{L} \\ 
	\leq & \mathbb{E} \norm{\bm{v}_k}^2 u_{96} + \mathbb{E} \norm{\bm{x}_k}^2 u_{88} + \mathbb{E} \norm{\tilde{\nabla}_k-\nabla f(\bm{x}_k)}^2 u_{93} + d u_{68} \\ 
	\leq & \mathbb{E} \norm{\bm{v}_k}^2 \max\left(0, u_{96}\right) + \mathbb{E} \norm{\bm{x}_k}^2 u_{88} + \mathbb{E} \norm{\tilde{\nabla}_k-\nabla f(\bm{x}_k)}^2 \max\left(0, u_{93}\right) + d u_{68}
	\end{aligned}\end{equation}
	
	The first inequality comes from Lipschitz condition.
	
	The second inequality comes from strongly convex condition of $f(\bm{x})$.
	
	The third inequality comes from Young's inequalities.
	$$\mathbb{E} \inner{\nabla f(\bm{x}_k)}{\bm{v}_k} \leq \frac{L \mathbb{E} \norm{\bm{v}_k}^2}{4} + \frac{\mathbb{E} \norm{\nabla f(\bm{x}_k)}^2}{L}$$
	$$\mathbb{E} \inner{\tilde{\nabla}_k-\nabla f(\bm{x}_k)}{\nabla f(\bm{x}_k)} \leq \frac{\delta \mathbb{E} \norm{\nabla f(\bm{x}_k)}^2}{2} + \frac{\mathbb{E} \norm{\tilde{\nabla}_k-\nabla f(\bm{x}_k)}^2}{2 \delta}$$
	
	The fourth inequality comes from Lipschitz condition.
	
	The fifth inequality comes from Young's inequalities.
	$$\mathbb{E} \inner{\bm{v}_k}{\bm{x}_k} \leq \frac{\mathbb{E} \norm{\bm{v}_k}^2}{4} + \mathbb{E} \norm{\bm{x}_k}^2$$
	$$\mathbb{E} \inner{\tilde{\nabla}_k-\nabla f(\bm{x}_k)}{\bm{x}_k} \leq \frac{L \mathbb{E} \norm{\bm{x}_k}^2}{2 \kappa} + \frac{\kappa \mathbb{E} \norm{\tilde{\nabla}_k-\nabla f(\bm{x}_k)}^2}{2 L}$$
	$$\mathbb{E} \inner{\tilde{\nabla}_k-\nabla f(\bm{x}_k)}{\bm{v}_k} \leq \frac{L \mathbb{E} \norm{\bm{v}_k}^2}{2} + \frac{\mathbb{E} \norm{\tilde{\nabla}_k-\nabla f(\bm{x}_k)}^2}{2 L}$$
	
	We apply Gronwall's inequality on \cref{eq:stage1_dE_proof} to finish the proof. 
\end{proof}

\begin{lemma}\label{le:stage1_momentum}
	\begin{equation}\label{eq:stage1_momentum_x}
		\max_k \mathbb{E} \norm{\bm{x}_k}^2 \leq \max_k \mathbb{E} \left( E(\bm{x}_k, \bm{v}_k) \right)
	\end{equation}
	\begin{equation}\label{eq:stage1_momentum_v}
		\max_k \mathbb{E} \norm{\bm{v}_k}^2 \leq 2 \max_k \mathbb{E} \left( E(\bm{x}_k, \bm{v}_k) \right)
	\end{equation}
	\begin{equation}\label{eq:stage1_momentum_f}
		\max_k \mathbb{E} \norm{\nabla f(\bm{x}_k)}^2 \leq L^{2} \max_k \mathbb{E} \norm{\bm{x}_k}^2
	\end{equation}
\end{lemma}
\begin{proof}[Proof of \cref{le:stage1_momentum}]
	These inequalities follows from definition of $E$ and Lipschitz condition.
\end{proof}

\begin{lemma}\label{le:stage1_theta}
	\begin{equation}\label{eq:stage1_theta}
		\max_k \mathbb{E} \norm{\tilde{\nabla}_k-\nabla f(\bm{x}_k)}^2 \leq \Theta \max_k Q_k
	\end{equation}
	where  $\Theta= \frac{M_1}{\rho_M} + \frac{M_2}{\rho_M \rho_F}$.
\end{lemma}
\begin{proof}[Proof of \cref{le:stage1_theta}]
	\begin{equation}
		\begin{aligned}
			\mathcal{M}_k \leq & M_1 Q_k+\mathcal{F}_k + (1 - \rho_M ) \mathcal{M}_{k-1} \\
			\leq & M_1 \sum_{i=0}^{k} (1-\rho_M)^i Q_{k-i} + \sum_{i=0}^{k} (1-\rho_M)^{k-i} \mathcal{F}_i
		\end{aligned}
	\end{equation}
	\begin{equation}
		M_1 \sum_{i=0}^{k} (1-\rho_M)^i Q_{k-i}  \leq \frac{M_1}{\rho_M} \max_k Q_k
	\end{equation}
	\begin{equation}
		\begin{aligned}
			\sum_{i=0}^{k} (1-\rho_M)^{k-i} \mathcal{F}_i \leq & M_2 \sum_{i=0}^{k} \sum_{l=0}^{i} (1-\rho_F)^{i-l} (1-\rho_M)^{k-i} Q_l \\
			\leq & \frac{M_2}{\rho_M \rho_F} \max_k Q_k
		\end{aligned}
	\end{equation}
\end{proof}
\begin{lemma}\label{le:stage1_Q}
	\begin{equation}\label{eq:stage1_Q}
	\begin{aligned}
		\max_k Q_k \leq & \max_k \mathbb{E} \norm{\bm{v}_k}^2 u_{127} + \max_k \mathbb{E} \norm{\bm{x}_k}^2 u_{124} \\
		& + \max_k \mathbb{E} \norm{\tilde{\nabla}_k-\nabla f(\bm{x}_k)}^2 u_{126} + d u_{105}
	\end{aligned}
	\end{equation}
\end{lemma}

\begin{proof}[Proof of \cref{le:stage1_Q}]
	\begin{equation}\begin{aligned}
	Q_k = & N \sum_{i=1}^{N} \mathbb{E} \norm{\nabla f_i (\bm{x}_{k+1}) - \nabla f_i (\bm{x}_{k})}^2 \\
	\leq & - 2 L^{2} \mathbb{E} \inner{\bm{x}_{k+1}}{\bm{x}_k} + L^{2} \mathbb{E} \norm{\bm{x}_k}^2 + L^{2} \mathbb{E} \norm{\bm{x}_{k+1}}^2 \\
	= & \mathbb{E} \inner{\tilde{\nabla}_k}{\bm{v}_k} u_{111} + \mathbb{E} \norm{\bm{v}_k}^2 u_{113} + \mathbb{E} \norm{\tilde{\nabla}_k}^2 u_{109} + d u_{105} \\ 
	= & \mathbb{E} \inner{\nabla f(\bm{x}_k)}{\bm{v}_k} u_{111} + \mathbb{E} \inner{\tilde{\nabla}_k-\nabla f(\bm{x}_k)}{\bm{v}_k} u_{111} \\
	& + \mathbb{E} \inner{\tilde{\nabla}_k-\nabla f(\bm{x}_k)}{\nabla f(\bm{x}_k)} u_{114} + \mathbb{E} \norm{\bm{v}_k}^2 u_{113} + \mathbb{E} \norm{\nabla f(\bm{x}_k)}^2 u_{109} \\
	& + \mathbb{E} \norm{\tilde{\nabla}_k-\nabla f(\bm{x}_k)}^2 u_{109} + d u_{105} \\ 
	\leq & \mathbb{E} \inner{\tilde{\nabla}_k-\nabla f(\bm{x}_k)}{\bm{v}_k} u_{111} + \mathbb{E} \norm{\bm{v}_k}^2 u_{123} + \mathbb{E} \norm{\nabla f(\bm{x}_k)}^2 u_{120} \\
	& + \mathbb{E} \norm{\tilde{\nabla}_k-\nabla f(\bm{x}_k)}^2 u_{118} + d u_{105} \\ 
	\leq & \mathbb{E} \inner{\tilde{\nabla}_k-\nabla f(\bm{x}_k)}{\bm{v}_k} u_{111} + \mathbb{E} \norm{\bm{v}_k}^2 u_{123} + \mathbb{E} \norm{\bm{x}_k}^2 u_{124} \\
	& + \mathbb{E} \norm{\tilde{\nabla}_k-\nabla f(\bm{x}_k)}^2 u_{118} + d u_{105} \\ 
	\leq & \mathbb{E} \norm{\bm{v}_k}^2 u_{127} + \mathbb{E} \norm{\bm{x}_k}^2 u_{124} + \mathbb{E} \norm{\tilde{\nabla}_k-\nabla f(\bm{x}_k)}^2 u_{126} + d u_{105}
	\end{aligned}\end{equation}

	The first inequality comes from Young's inequalities.
	$$\mathbb{E} \inner{\nabla f(\bm{x}_k)}{\bm{v}_k} \leq \frac{L \mathbb{E} \norm{\bm{v}_k}^2}{4} + \frac{\mathbb{E} \norm{\nabla f(\bm{x}_k)}^2}{L}$$
	$$\mathbb{E} \inner{\tilde{\nabla}_k-\nabla f(\bm{x}_k)}{\nabla f(\bm{x}_k)} \leq \frac{\delta \mathbb{E} \norm{\nabla f(\bm{x}_k)}^2}{2} + \frac{\mathbb{E} \norm{\tilde{\nabla}_k-\nabla f(\bm{x}_k)}^2}{2 \delta}$$
	
	The second inequality comes from Lipschitz condition.
	
	The third inequality comes from Young's inequalities.
	$$\mathbb{E} \inner{\tilde{\nabla}_k-\nabla f(\bm{x}_k)}{\bm{v}_k} \leq \frac{L \mathbb{E} \norm{\bm{v}_k}^2}{2} + \frac{\mathbb{E} \norm{\tilde{\nabla}_k-\nabla f(\bm{x}_k)}^2}{2 L}$$
\end{proof}

The full expression of terms $u_i$ is as follows.
$$u_{18} = \frac{\delta}{10 \kappa} - 1$$
$$u_{24} = \frac{\delta}{5 L \kappa} - \frac{2}{L}$$
$$u_{27} = \frac{\delta}{5 \kappa} - 2$$
$$u_{32} = \frac{3 \delta}{10 \kappa} - 1$$
$$u_{33} = \delta - 3$$
$$u_{41} = 3 \delta^{2} - 2 \delta + 3$$
$$u_{42} = u_{41} e^{2 \delta}$$
$$u_{47} = \frac{u_{33} e^{- \delta}}{8 L^{2}} + \frac{u_{42} e^{- 2 \delta}}{16 L^{2}} + \frac{3 e^{- 2 \delta}}{16 L^{2}}$$
$$u_{49} = 3 \delta - 1$$
$$u_{52} = - \frac{u_{49}}{2 L} - \frac{e^{- \delta}}{2 L}$$
$$u_{53} = - \frac{\delta}{2 L} - \frac{1}{2 L} + \frac{e^{- \delta}}{2 L}$$
$$u_{54} = \delta - 1$$
$$u_{55} = u_{54} e^{\delta}$$
$$u_{56} = - \frac{u_{55} e^{- \delta}}{2 L^{2}} - \frac{e^{- \delta}}{2 L^{2}}$$
$$u_{57} = \delta e^{\delta} - u_{49} e^{2 \delta} - 4 e^{\delta} + 3$$
$$u_{60} = \frac{u_{57} e^{- 2 \delta}}{4 L}$$
$$u_{61} = 3 - e^{- \delta}$$
$$u_{62} = 1 + e^{- \delta}$$
$$u_{63} = \frac{1}{L} - \frac{e^{- \delta}}{L}$$
$$u_{64} = \frac{\delta}{10 \kappa} - \frac{1}{4} - \frac{e^{- \delta}}{2} + \frac{3 e^{- 2 \delta}}{4}$$
$$u_{68} = \frac{3 \delta}{2 L} - \frac{1}{4 L} + \frac{e^{- 2 \delta}}{4 L}$$
$$u_{70} = \frac{3 \delta^{2}}{8 L^{2}} - \frac{3 \delta}{4 L^{2}} + \frac{\delta e^{- \delta}}{4 L^{2}} + \frac{7}{8 L^{2}} - \frac{5 e^{- \delta}}{4 L^{2}} + \frac{3 e^{- 2 \delta}}{8 L^{2}}$$
$$u_{71} = 3 \delta^{2} e^{2 \delta} - 10 \delta e^{2 \delta} + 2 \delta e^{\delta} + 11 e^{2 \delta} - 14 e^{\delta} + 3$$
$$u_{72} = \frac{u_{71} e^{- 2 \delta}}{16 L^{2}}$$
$$u_{73} = - \frac{3 \delta}{4 L} + \frac{\delta e^{- \delta}}{4 L} + \frac{5}{4 L} - \frac{2 e^{- \delta}}{L} + \frac{3 e^{- 2 \delta}}{4 L}$$
$$u_{77} = \left|{3 \delta^{2} e^{2 \delta} - 6 \delta e^{2 \delta} + 2 \delta e^{\delta} + 7 e^{2 \delta} - 10 e^{\delta} + 3}\right|$$
$$u_{78} = 2 \delta u_{33} e^{\delta} + u_{77}$$
$$u_{80} = \frac{u_{42} e^{- 2 \delta}}{16 L^{2}} + \frac{3 e^{- 2 \delta}}{16 L^{2}} + \frac{u_{78} e^{- 2 \delta}}{16 L^{2} \delta}$$
$$u_{85} = \left|{- 3 \delta e^{2 \delta} + \delta e^{\delta} + 5 e^{2 \delta} - 8 e^{\delta} + 3}\right|$$
$$u_{86} = \frac{\delta u_{77} e^{- 2 \delta}}{16 L^{2}} + \frac{u_{71} e^{- 2 \delta}}{16 L^{2}} + \frac{u_{85} e^{- 2 \delta}}{4 L^{2}}$$
$$u_{87} = \frac{\delta}{10 \kappa} + \frac{u_{85} e^{- 2 \delta}}{16} - \frac{1}{4} - \frac{e^{- \delta}}{2} + \frac{3 e^{- 2 \delta}}{4}$$
$$u_{88} = L^{2} \max\left(0, u_{86}\right)$$
$$u_{89} = - \frac{7 \delta}{10 \kappa} + u_{88}$$
$$u_{90} = 2 \left|{u_{57}}\right|$$
$$u_{93} = \frac{\delta \kappa}{2 L^{2}} + \frac{u_{41}}{16 L^{2}} + \frac{u_{90} e^{- 2 \delta}}{16 L^{2}} + \frac{3 e^{- 2 \delta}}{16 L^{2}} + \frac{u_{78} e^{- 2 \delta}}{16 L^{2} \delta}$$
$$u_{96} = \frac{3 \delta}{20 \kappa} + \frac{u_{85} e^{- 2 \delta}}{16} + \frac{u_{90} e^{- 2 \delta}}{16} - \frac{1}{4} - \frac{e^{- \delta}}{2} + \frac{3 e^{- 2 \delta}}{4}$$
$$u_{100} = \frac{10 \kappa \max\left(0, u_{93}\right)}{\delta}$$
$$u_{101} = \frac{10 \kappa \max\left(0, u_{96}\right)}{\delta}$$
$$u_{102} = \frac{10 \kappa u_{88}}{\delta}$$
$$u_{104} = \frac{15 \kappa}{L} - \frac{5 \kappa}{2 L \delta} + \frac{5 \kappa e^{- 2 \delta}}{2 L \delta}$$
$$u_{105} = \frac{L \delta}{2} - \frac{3 L}{4} + L e^{- \delta} - \frac{L e^{- 2 \delta}}{4}$$
$$u_{106} = \delta^{2} e^{2 \delta} - 2 \delta e^{2 \delta} + e^{2 \delta}$$
$$u_{107} = u_{106} + 2 u_{55} + 1$$
$$u_{108} = u_{107} e^{- 2 \delta}$$
$$u_{109} = \frac{u_{108}}{16}$$
$$u_{110} = \delta e^{\delta} - u_{54} e^{2 \delta} - 2 e^{\delta} + 1$$
$$u_{111} = \frac{L u_{110} e^{- 2 \delta}}{4}$$
$$u_{112} = e^{2 \delta} - 2 e^{\delta} + 1$$
$$u_{113} = \frac{L^{2} u_{112} e^{- 2 \delta}}{4}$$
$$u_{114} = \frac{u_{108}}{8}$$
$$u_{115} = \left|{u_{107}}\right|$$
$$u_{116} = \delta u_{106} + 2 \delta u_{55} + u_{115}$$
$$u_{118} = \frac{e^{- 2 \delta}}{16} + \frac{u_{116} e^{- 2 \delta}}{16 \delta}$$
$$u_{119} = \left|{u_{110}}\right|$$
$$u_{120} = \frac{\delta u_{115} e^{- 2 \delta}}{16} + \frac{u_{107} e^{- 2 \delta}}{16} + \frac{u_{119} e^{- 2 \delta}}{4}$$
$$u_{121} = 4 e^{2 \delta} - 8 e^{\delta} + 4$$
$$u_{123} = \frac{L^{2} u_{119} e^{- 2 \delta}}{16} + \frac{L^{2} u_{121} e^{- 2 \delta}}{16}$$
$$u_{124} = L^{2} \max\left(0, u_{120}\right)$$
$$u_{125} = \left|{- \delta e^{2 \delta} + \delta e^{\delta} + u_{112}}\right|$$
$$u_{126} = \frac{u_{125} e^{- 2 \delta}}{8} + \frac{e^{- 2 \delta}}{16} + \frac{u_{116} e^{- 2 \delta}}{16 \delta}$$
$$u_{127} = \frac{L^{2} u_{121} e^{- 2 \delta}}{16} + \frac{3 L^{2} u_{125} e^{- 2 \delta}}{16}$$
$$u_{129} = 5 \delta e^{\delta} + 5 e^{2 \delta}$$
\begin{equation*}
	\begin{aligned}
	u_{132} = & \max\left(0, - \frac{33 \kappa}{8} + \frac{3 \kappa u_{129} e^{- 2 \delta}}{8 \delta} - \frac{15 \kappa e^{- \delta}}{\delta} + \frac{105 \kappa e^{- 2 \delta}}{8 \delta}, \right. \\
	&\left. \frac{27 \kappa}{8} - \frac{\kappa u_{129} e^{- 2 \delta}}{8 \delta} - \frac{5 \kappa e^{- \delta}}{\delta} + \frac{45 \kappa e^{- 2 \delta}}{8 \delta}, \right. \\
	&\left. - \frac{3 \kappa}{8} + \frac{5 \kappa e^{- \delta}}{8} - \frac{35 \kappa}{8 \delta} - \frac{5 \kappa e^{- \delta}}{\delta} + \frac{75 \kappa e^{- 2 \delta}}{8 \delta}, \right. \\
	&\left. \frac{57 \kappa}{8} - \frac{15 \kappa e^{- \delta}}{8} - \frac{55 \kappa}{8 \delta} + \frac{5 \kappa e^{- \delta}}{\delta} + \frac{15 \kappa e^{- 2 \delta}}{8 \delta}\right)
	\end{aligned}
\end{equation*}
\begin{equation*}
\begin{aligned}
u_{133} = & \max\left(- \frac{3 L^{2} \delta}{16} + \frac{3 L^{2} \delta e^{- \delta}}{16} + \frac{7 L^{2}}{16} - \frac{7 L^{2} e^{- \delta}}{8} + \frac{7 L^{2} e^{- 2 \delta}}{16}, \right. \\
&\left. \frac{3 L^{2} \delta}{16} - \frac{3 L^{2} \delta e^{- \delta}}{16} + \frac{L^{2}}{16} - \frac{L^{2} e^{- \delta}}{8} + \frac{L^{2} e^{- 2 \delta}}{16}\right)
\end{aligned}
\end{equation*}
$$u_{134} = \frac{4 \kappa^{2}}{L^{2}} + \frac{5 \kappa}{L^{2}}$$
$$u_{135} = \frac{5 u_{134}}{4}$$
$$u_{137} = \delta + 2$$
$$u_{138} = \frac{5 \kappa u_{137}}{L}$$
$$u_{146} = \frac{5 u_{134}}{2}$$
$$u_{147} = \frac{10 \kappa u_{137}}{L}$$
\begin{equation*}
\begin{aligned}
u_{158} = & \frac{25 \Theta \delta^{4} \kappa^{3}}{L} + \frac{125 \Theta \delta^{4} \kappa^{2}}{4 L} + \frac{150 \Theta \delta^{3} \kappa^{3}}{L} + \frac{1145 \Theta \delta^{3} \kappa^{2}}{6 L} + \frac{25 \Theta \delta^{3} \kappa}{6 L} \\
& + \frac{200 \Theta \delta^{2} \kappa^{3}}{L} + \frac{250 \Theta \delta^{2} \kappa^{2}}{L} + \frac{10 \delta \kappa}{L} + \frac{20 \kappa}{L}
\end{aligned}
\end{equation*}
$$u_{159} = 30 \Theta \delta^{3} \kappa^{2} + \frac{75 \Theta \delta^{3} \kappa}{2} + 120 \Theta \delta^{2} \kappa^{2} + 150 \Theta \delta^{2} \kappa + 12$$
$$u_{160} = \frac{5 L \Theta \delta^{4} \kappa}{2} + 15 L \Theta \delta^{3} \kappa + \frac{L \Theta \delta^{3}}{3} + 20 L \Theta \delta^{2} \kappa$$
$$u_{161} = 3 L^{2} \Theta \delta^{3} + 12 L^{2} \Theta \delta^{2}$$

\section{Proof of \Cref{co:saga,co:svrg,co:sarah,co:sarge}}
	According to Proposition 2-4 in \cite{Driggs2019}, 
	the SAGA gradient estimator satisfies MSEB property with $M_1 = 3 N/b^2, \rho_M = \frac{b}{2 N}, M_2 = 0, \rho_F=1$. 
	The SVRG gradient estimator satisfies MSEB property with $M_1 = 3p/b, \rho_M = \frac{1}{2 p}, M_2 = 0, \rho_F=1$. 
	the SARAH gradient estimator satisfies MSEB property with $M_1 = 1, \rho_M = 1/p, M_2 = 0, \rho_F=1$. 
	the SARGE gradient estimator satisfies MSEB property with $M_1 = 12, \rho_M = \frac{b}{2N}, M_2 = (27+12 b)/N, \rho_F=\frac{b}{2N}$. 
	Applying these parameters to \cref{th:main} would lead to \cref{co:saga,co:svrg,co:sarah,co:sarge}.

%% file: hmc.bbl
\begin{thebibliography}{28}
\providecommand{\natexlab}[1]{#1}
\providecommand{\url}[1]{\texttt{#1}}
\expandafter\ifx\csname urlstyle\endcsname\relax
  \providecommand{\doi}[1]{doi: #1}\else
  \providecommand{\doi}{doi: \begingroup \urlstyle{rm}\Url}\fi

\bibitem[Chiang and Hwang(1987)]{Chiang1987}
Tzuu-Shuh Chiang and Chii-Ruey Hwang.
\newblock Diffusion for global optimization in rn.
\newblock \emph{SIAM J. Control Optim.}, 25\penalty0 (3):\penalty0 737–753,
  May 1987.

\bibitem[Kloeden and Platen(2013)]{Kloeden2013}
Peter~E Kloeden and Eckhard Platen.
\newblock \emph{Numerical solution of stochastic differential equations},
  volume~23.
\newblock Springer Science \& Business Media, 2013.

\bibitem[Dalalyan(2017{\natexlab{a}})]{Dalalyan2017a}
Arnak~S. Dalalyan.
\newblock Theoretical guarantees for approximate sampling from a smooth and
  log-concave density.
\newblock \emph{J. R. Stat. Soc. B}, 79:\penalty0 651–676,
  2017{\natexlab{a}}.

\bibitem[Durmus and Moulines(2016)]{Durmus2016}
Alain Durmus and Eric Moulines.
\newblock Sampling from strongly log-concave distributions with the unadjusted
  langevin algorithm.
\newblock \emph{arXiv preprint arXiv:1605.01559}, 5, 2016.

\bibitem[Dalalyan(2017{\natexlab{b}})]{Dalalyan2017}
Arnak Dalalyan.
\newblock Further and stronger analogy between sampling and optimization:
  Langevin monte carlo and gradient descent.
\newblock In \emph{Proceedings of the 2017 Conference on Learning Theory},
  volume~65, pages 678--689, 07--10 Jul 2017{\natexlab{b}}.

\bibitem[Dalalyan and Karagulyan(2019)]{Dalalyan2019}
Arnak~S Dalalyan and Avetik Karagulyan.
\newblock User-friendly guarantees for the langevin monte carlo with inaccurate
  gradient.
\newblock \emph{Stochastic Processes and their Applications}, 129\penalty0
  (12):\penalty0 5278--5311, 2019.

\bibitem[Durmus et~al.(2017)Durmus, Moulines, et~al.]{Durmus2017}
Alain Durmus, Eric Moulines, et~al.
\newblock Nonasymptotic convergence analysis for the unadjusted langevin
  algorithm.
\newblock \emph{The Annals of Applied Probability}, 27\penalty0 (3):\penalty0
  1551--1587, 2017.

\bibitem[Cheng and Bartlett(2018)]{Cheng2018}
Xiang Cheng and Peter Bartlett.
\newblock Convergence of langevin mcmc in kl-divergence.
\newblock \emph{Proceedings of Machine Learning Research, Volume 83:
  Algorithmic Learning Theory}, pages 186--211, 2018.

\bibitem[Duane et~al.(1987)Duane, Kennedy, Pendleton, and Roweth]{Duane1987}
Simon Duane, Anthony~D Kennedy, Brian~J Pendleton, and Duncan Roweth.
\newblock Hybrid monte carlo.
\newblock \emph{Physics letters B}, 195\penalty0 (2):\penalty0 216--222, 1987.

\bibitem[Neal et~al.(2011)]{Neal2011}
Radford~M Neal et~al.
\newblock Mcmc using hamiltonian dynamics.
\newblock \emph{Handbook of markov chain monte carlo}, 2\penalty0
  (11):\penalty0 2, 2011.

\bibitem[Cheng et~al.(2018)Cheng, Chatterji, Bartlett, and Jordan]{Cheng2018a}
Xiang Cheng, Niladri~S. Chatterji, Peter~L. Bartlett, and Michael~I. Jordan.
\newblock Underdamped langevin mcmc: A non-asymptotic analysis.
\newblock In \emph{Proceedings of the 31st Conference On Learning Theory},
  volume~75, pages 300--323, 2018.

\bibitem[Ma et~al.(2019)Ma, Chatterji, Cheng, Flammarion, Bartlett, and
  Jordan]{Ma2019}
Yi-An Ma, Niladri Chatterji, Xiang Cheng, Nicolas Flammarion, Peter Bartlett,
  and Michael~I Jordan.
\newblock Is there an analog of nesterov acceleration for mcmc?
\newblock \emph{arXiv preprint arXiv:1902.00996}, 2019.

\bibitem[Welling and Teh(2011)]{Welling2011}
Max Welling and Yee~Whye Teh.
\newblock Bayesian learning via stochastic gradient langevin dynamics.
\newblock In \emph{Proceedings of the 28th International Conference on
  International Conference on Machine Learning}, page 681–688, 2011.

\bibitem[Ma et~al.(2015)Ma, Chen, and Fox]{Ma2015}
Yi-An Ma, Tianqi Chen, and Emily Fox.
\newblock A complete recipe for stochastic gradient mcmc.
\newblock In \emph{Advances in Neural Information Processing Systems}, pages
  2917--2925, 2015.

\bibitem[Chen et~al.(2014)Chen, Fox, and Guestrin]{Chen2014}
Tianqi Chen, Emily Fox, and Carlos Guestrin.
\newblock Stochastic gradient hamiltonian monte carlo.
\newblock In \emph{International conference on machine learning}, pages
  1683--1691, 2014.

\bibitem[Dubey et~al.(2016)Dubey, Reddi, Williamson, Poczos, Smola, and
  Xing]{Dubey2016}
Kumar~Avinava Dubey, Sashank~J Reddi, Sinead~A Williamson, Barnabas Poczos,
  Alexander~J Smola, and Eric~P Xing.
\newblock Variance reduction in stochastic gradient langevin dynamics.
\newblock In \emph{Advances in neural information processing systems}, pages
  1154--1162, 2016.

\bibitem[Li et~al.(2019)Li, Zhang, and Li]{Li2019}
Zhize Li, Tianyi Zhang, and Jian Li.
\newblock Stochastic gradient hamiltonian monte carlo with variance reduction
  for bayesian inference.
\newblock \emph{Machine Learning}, 108:\penalty0 1701--1727, 2019.

\bibitem[Baker et~al.(2019)Baker, Fearnhead, Fox, and Nemeth]{Baker2019}
Jack Baker, Paul Fearnhead, Emily~B Fox, and Christopher Nemeth.
\newblock Control variates for stochastic gradient mcmc.
\newblock \emph{Statistics and Computing}, 29\penalty0 (3):\penalty0 599--615,
  2019.

\bibitem[Chatterji et~al.(2018)Chatterji, Flammarion, Ma, Bartlett, and
  Jordan]{Chatterji2018}
Niladri~S. Chatterji, Nicolas Flammarion, Yi{-}An Ma, Peter~L. Bartlett, and
  Michael~I. Jordan.
\newblock On the theory of variance reduction for stochastic gradient monte
  carlo.
\newblock In \emph{{ICML} 2018}, volume~80, pages 763--772, 2018.

\bibitem[Zou et~al.(2018)Zou, Xu, and Gu]{Zou2018}
Difan Zou, Pan Xu, and Quanquan Gu.
\newblock Stochastic variance-reduced hamilton monte carlo methods.
\newblock In \emph{Proceedings of the 35th International Conference on Machine
  Learning}, volume~80, pages 6023--6032, 2018.

\bibitem[Mangoubi and Vishnoi(2018)]{Mangoubi2018}
Oren Mangoubi and Nisheeth Vishnoi.
\newblock Dimensionally tight bounds for second-order hamiltonian monte carlo.
\newblock In \emph{Advances in neural information processing systems}, pages
  6027--6037, 2018.

\bibitem[Chen et~al.(2019)Chen, Chen, Dong, Peng, and Wang]{Chen2019}
Yi~Chen, Jinglin Chen, Jing Dong, Jian Peng, and Zhaoran Wang.
\newblock Accelerating nonconvex learning via replica exchange langevin
  diffusion.
\newblock In \emph{7th International Conference on Learning Representations},
  2019.

\bibitem[Deng et~al.(2020)Deng, Feng, Gao, Liang, and Lin]{Deng2020}
Wei Deng, Qi~Feng, Liyao Gao, Faming Liang, and Guang Lin.
\newblock Non-convex learning via replica exchange stochastic gradient mcmc.
\newblock In \emph{Proceedings of Machine Learning and Systems 2020}, pages
  2781--2790, 2020.

\bibitem[Defazio et~al.(2014)Defazio, Bach, and Lacoste-Julien]{Defazio2014}
Aaron Defazio, Francis Bach, and Simon Lacoste-Julien.
\newblock Saga: A fast incremental gradient method with support for
  non-strongly convex composite objectives.
\newblock In \emph{Advances in neural information processing systems}, pages
  1646--1654, 2014.

\bibitem[Johnson and Zhang(2013)]{Johnson2013}
Rie Johnson and Tong Zhang.
\newblock Accelerating stochastic gradient descent using predictive variance
  reduction.
\newblock In \emph{Advances in neural information processing systems}, pages
  315--323, 2013.

\bibitem[Nguyen et~al.(2017)Nguyen, Liu, Scheinberg, and
  Tak{\'a}{\v{c}}]{Nguyen2017}
Lam~M Nguyen, Jie Liu, Katya Scheinberg, and Martin Tak{\'a}{\v{c}}.
\newblock Sarah: A novel method for machine learning problems using stochastic
  recursive gradient.
\newblock In \emph{Proceedings of the 34th International Conference on Machine
  Learning-Volume 70}, pages 2613--2621, 2017.

\bibitem[Driggs et~al.(2019)Driggs, Ehrhardt, and Sch{\"o}nlieb]{Driggs2019}
Derek Driggs, Matthias~J Ehrhardt, and Carola-Bibiane Sch{\"o}nlieb.
\newblock Accelerating variance-reduced stochastic gradient methods.
\newblock \emph{arXiv preprint arXiv:1910.09494}, 2019.

\bibitem[Chang and Lin(2011)]{Chang2011}
Chih-Chung Chang and Chih-Jen Lin.
\newblock {LIBSVM}: A library for support vector machines.
\newblock \emph{ACM Transactions on Intelligent Systems and Technology},
  2:\penalty0 27:1--27:27, 2011.
\newblock Software available at \url{http://www.csie.ntu.edu.tw/~cjlin/libsvm}.

\end{thebibliography}
